\documentclass{article}
\usepackage{lineno}

\PassOptionsToPackage{numbers, compress}{natbib}

\usepackage[preprint]{neurips_2026}

\usepackage{silence}
\usepackage[utf8]{inputenc}
\usepackage[T1]{fontenc}
\usepackage{hyperref}
\usepackage{url}
\usepackage{booktabs}
\usepackage{multirow}
\usepackage{amsfonts}
\usepackage{nicefrac}
\usepackage[final]{microtype}
\usepackage{xcolor}
\usepackage{xspace}
\usepackage{amsmath}
\usepackage{amssymb}
\usepackage{amsthm}
\usepackage{algorithm}
\usepackage{algpseudocode}
\usepackage{caption}
\usepackage{adjustbox}
\usepackage{titlesec}
\usepackage{enumitem}
\usepackage[most]{tcolorbox}

\tcbuselibrary{skins}

\definecolor{darkgreen}{RGB}{0,130,0}
\DeclareFontShape{T1}{ptm}{m}{scit}{<->ssub * ptm/m/sc}{}

\newcommand{\sys}{\textup{\textsc{Soma-SQL}}\xspace}
\newcommand{\dr}[1]{\textcolor{purple}{DR: #1}}
\newcommand{\ssb}[1]{\textcolor{magenta}{SSB: #1}}
\newcommand{\ssa}[1]{\textcolor{green}{SSA: #1}}

\newtheorem{proposition}{Proposition}

\newtheorem*{problem*}{Problem definition}

\newtcolorbox{promptbox}[2][]{%
  enhanced,
  breakable,
  colback=gray!10,
  colframe=gray!65,
  colbacktitle=gray!70,
  coltitle=white,
  fonttitle=\bfseries,
  title={#2},
  boxrule=0.8pt,
  arc=4pt,
  left=8pt,
  right=8pt,
  top=8pt,
  bottom=8pt,
  boxed title style={sharp corners},
  halign=flush left,
  #1
}

\titlespacing*{\subsection}{0pt}{0.8ex}{0.2ex}
\title{\sys: Resolving Multi-Source Ambiguity in NL-to-SQL via Synthetic Log and Execution Probing}

\author{%
\begin{tabular}{c}
\textbf{Sai Ashish Somayajula\textsuperscript{*}},
\textbf{Marianne Menglin Liu\textsuperscript{*}},
\textbf{Chuan Lei},
\textbf{Fjona Parllaku}, \\
\textbf{Daniel Garcia},
\textbf{Rongguang Wang},
\textbf{Syed Fahad Allam Shah},
\textbf{Ankan Bansal}, \\
\textbf{Sujeeth Bharadwaj},
\textbf{Tao Sheng},
\textbf{Sujith Ravi},
\textbf{Dan Roth} \\
{\normalfont Oracle AI} \\
{\normalfont\texttt{\{ashish.somayajula, marianne.liu, dan.roth\}@oracle.com}}
\end{tabular}
}

\begin{document}

\maketitle

\begin{center}
{\small $^*$Equal contribution.}
\end{center}

\begin{abstract}
Natural language interfaces to databases aim to translate user questions into executable SQL, yet remain brittle in real-world settings where questions are underspecified and schemas are large and ambiguous.
Ambiguity across user questions, database schemas, and model interpretations are central failure modes in NL2SQL, leading to misaligned intent, incorrect schema grounding, and erroneous SQL generation. 
Existing approaches rely on human clarification or treat ambiguity as a schema representation problem, but these do not scale nor resolve ambiguity autonomously.
We propose \sys to automatically resolve ambiguity via targeted synthetic query log and ambiguity-driven probing.
\sys constructs synthetic query log to ground schema interpretation and guide candidate SQL generation; it then executes targeted probing queries, driven by a structured ambiguity taxonomy and candidate disagreements, to produce disambiguation evidence for final SQL selection and repair. 
This active approach to ambiguity discovery and resolution generalizes across unseen schemas and query distributions without human-in-the-loop.
Experiments on six public benchmarks demonstrate that \sys improves execution accuracy by 13.0\% on average over state-of-the-art baselines, with gains of up to 16.7\% on ambiguous questions. 
\end{abstract}

\section{Introduction}
\label{sec:intro}

Translating natural language (NL) questions into SQL queries (NL2SQL) is a long-standing problem at the intersection of semantic parsing, pragmatics, and database systems. 
Despite recent advances in large language models (LLMs), NL2SQL systems~\cite{talaei2024chess,pourreza2024,deng2025reforce,wang2025autolink} remain brittle under the ambiguity and scale of enterprise deployments. 
Real-world databases are large, complex, and often poorly documented, with schemas spanning hundreds of tables and ambiguous naming conventions~\cite{lei2024spider,floratou2024nl2sql,luo2025nl2sql}. 
User questions are frequently underspecified~\cite{saparina2024ambrosia,vaidya2025odin,ding2026}, omitting aggregation metrics, filters, or join conditions, while implicitly assuming the system shares their domain knowledge and business logic.
As a result, while modern NL2SQL agents~\cite{wang2025autolink,agenticdata} perform well in idealized settings with clean schemas and well-specified queries, they continue to struggle against realistic enterprise databases with large, complex schemas and abstract, business-driven questions.

{\bf Motivating example.} 
Consider a sales analyst querying an enterprise database: ``\textit{What were the top-performing regions in New York last quarter?}'' The phrase \textit{top-performing} could refer to total revenue, growth rate, or number of closed deals, each requiring a fundamentally different aggregation.
\textit{Last quarter} may refer to the calendar quarter, fiscal quarter, or a rolling 90-day window.
Even \textit{regions} is underspecified when the database contains multiple possible geographic or business groupings, such as sales territories, geographic regions, or market segments.
\textit{New York} may refer to different entities (e.g., \textit{New York City} vs.\ \textit{New York State}) and may also appear in different surface forms (e.g., \textit{NY} or \textit{NYC}). A system must silently commit to certain assumptions for each ambiguity, producing SQL queries that are syntactically valid yet return results that do not match the analyst's intent, a failure that may not surface until downstream decisions have already been made. 

\begin{figure}[t]
    \vspace{-10pt}
    \centering
    \includegraphics[width=0.9\linewidth]{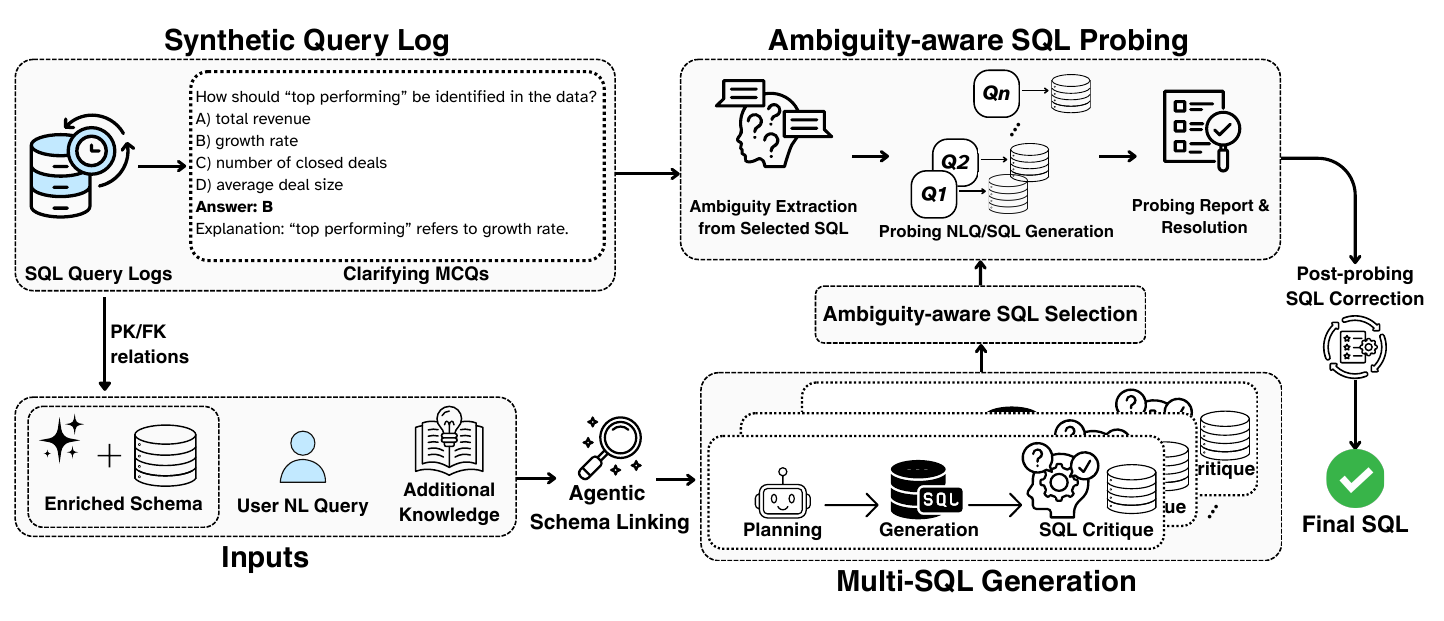}
    \vspace{-10pt}
    \caption{An overview of \sys consisting of (a) Synthetic Query Log Construction for Disambiguation, (b) Ambiguity-aware SQL Probing and Correction, and (c) Multi-SQL Generation.}
    \label{fig:archi}
\end{figure}

{\bf Challenges.} 
These limitations manifest as three ambiguity-driven failure modes. 
First, {\it query-level ambiguity} arises when user intent is underspecified. The desired aggregation, filter condition, or temporal granularity is left implicit, forcing the system to guess among equally valid interpretations.
Second, {\it schema-level ambiguity} occurs when overlapping or poorly named schema elements obscure which table, column, or value a query phrase refers to.
For instance, ``{\it revenue}'' may map to multiple columns with different semantics in the database. 
Third, {\it contextual ambiguity} emerges when questions presuppose domain knowledge or business logic absent from both the question and the schema, such as metric definitions, organizational hierarchies, or implicit filtering conventions, leaving the model without sufficient grounding to reason correctly.

{\bf State-of-the-art approaches.} 
While some recent works explicitly target ambiguity in NL2SQL, each addresses only a narrow facet of the problem. 
One line of work resolves ambiguity through human-in-the-loop clarification~\cite{saparina2024ambrosia,vaidya2025odin,ding2026}, prompting users to select among candidate interpretations. 
Though effective in controlled settings, these approaches require synchronous user interaction at inference time and fail silently when ambiguities escape detection. 
A second line of work~\cite{wang2023know,desaiatomsql} treats ambiguity as a byproduct of poor schema representation, improving grounding through richer schema encoding or retrieval-augmented context, but cannot distinguish a genuinely underspecified query from one that is merely hard to ground.
A fundamental gap remains: no existing system reasons explicitly about why a query is ambiguous, which type of ambiguity is present, or how to resolve it autonomously, nor do they exploit the semantic divergence across multiple plausible SQL interpretations as a diagnostic signal for unresolved intent.

{\bf \sys overview.} We propose \sys (\underline{S}ynthetic Query Log and Pr\underline{o}bing for \underline{M}ulti-source \underline{A}mbiguity Resolution), a generalizable NL2SQL system that resolves ambiguity without human-in-the-loop (shown in Fig.~\ref{fig:archi}). 
Given an NL question, \sys first exploits query log generated synthetically to enrich schema understanding with ambiguity-aware supervision.
After agentic schema linking~\cite{wang2025autolink}, a multi-SQL generation workflow produces diagnostic SQL candidates; an ambiguity-aware SQL selector identifies a seed SQL for refinement, while cross-candidate disagreements expose the unresolved semantic ambiguities.
\sys then maps these conflicts, guided by a predefined ambiguity taxonomy, into structured ambiguity dimensions, and probes the database with targeted SQL queries to gather evidence for each.
The resulting probing report resolves the ambiguity dimensions and repairs the seed SQL to produce the final SQL with data-grounded evidence.


\textbf{Contributions.}
Our contributions can be summarized as follows: 
(1) we present \sys, a generalizable NL2SQL system that tackles multi-source ambiguity through iterative disambiguation and execution-grounded reasoning, producing semantically correct SQL without human-in-the-loop;
(2) \sys leverages synthetic query log constructed from SQL queries without corresponding NL questions, as is typical in enterprise settings, as a unified mechanism for ambiguity grounding, transforming SQL variations into structured signals for schema enrichment, semantic interpretation, and probing guidance;
(3) we develop an execution-grounded probing framework in which candidate SQL disagreements and a structured ambiguity taxonomy jointly guide targeted probe generation, producing concrete resolution evidence for SQL refinement; and
(4) 
experimental results on six public benchmarks show that \sys improves execution accuracy by 13.0\% on average over state-of-the-art baselines, with gains of up to 16.7\% on ambiguous questions.
\vspace{-4pt}
\section{Related Works}
\label{sec:related}

{\bf NL2SQL systems.}
LLMs have dramatically advanced NL2SQL, with methods such as DIN-SQL~\cite{pourreza2023dinsql}, DAIL-SQL~\cite{gao2023dail}, and MAC-SQL~\cite{wang2025mac} demonstrating that decomposed prompting and multi-agent collaboration yield strong results on benchmarks like Spider~\cite{yu2018spider} and BIRD~\cite{li2023bird}. More recent agentic frameworks further improve schema linking and error correction. ReFoRCE~\cite{deng2025reforce} combines schema compression with execution-guided exploration; AutoLink~\cite{wang2025autolink} performs iterative schema linking without requiring the full schema upfront; PV-SQL~\cite{tian2026pvsql} alternates between probing queries and rule-based verification for iterative refinement; and SQLens~\cite{gong2026sqlens} uses AST-level comparisons to localize semantic errors beyond LLM self-reflection. DeepEye-SQL~\cite{li2025deepeye} decomposes NL2SQL into schema linking, reasoning, implementation, debugging, and final selection. 
EnrichIndex~\cite{chen2025enrichindex} leverages LLMs offline to construct semantically enriched indices, highlighting the effectiveness of structured signals for downstream tasks. In contrast, BIRD-INTERACT~\cite{huo2025birdinteract} exposes the limitations of current systems in multi-turn, interactive settings. These works, however, largely assume well-formed input and offer no principled mechanism for handling query ambiguity.

{\bf Ambiguity in NL2SQL.}
AMBROSIA~\cite{saparina2024ambrosia} benchmarks three types of ambiguity in NL questions (i.e., scope, attachment, and vagueness) showing that even frontier LLMs fail to resolve them reliably. Wang et al.~\cite{wang2023know} detect and explain ambiguous or unanswerable queries via a Detecting-Then-Explaining framework trained on counterfactually generated data. ODIN~\cite{vaidya2025odin}, AtomSQL~\cite{desaiatomsql}, and AmbiSQL~\cite{ding2026} address schema and query ambiguity through candidate generation, user feedback, or clarification questions.
Our approach differs by learning ambiguity patterns from execution feedback and by automatically resolving ambiguities without human intervention. We discuss the broader literature on ambiguity resolution in NL beyond NL2SQL in Appendix~\ref{app:more_related}.

\section{Methodology}
\label{sec:method}


\subsection{Problem Formulation}
\label{sec:method:problem}

Let $S$ denote a database schema, $D$ a database instance, and $x$ a natural language question. An NL2SQL system generates an executable SQL query\footnote{SQL is a family of dialects, including SQLite, BigQuery SQL, Snowflake SQL, and Oracle SQL; and our solution is applicable to all of them. We refer to an NL question as a ``question'' and its corresponding SQL as a ``query''.} $q \in \mathcal{Q}(S)$, where $\mathcal{Q}(S)$ denotes the set of valid SQL queries over $S$. 
Since different SQL queries can produce the same result when executed, we treat the target SQL as one representative among all queries with equivalent execution semantics. 
In practice, enterprise queries are often underspecified: users may omit the intended metric, filter, time window, aggregation level, join path, or business convention.
For example, ``{\it top-performing regions in New York last quarter}'' may depend on the ranking metric (e.g., revenue, growth, or deal count), the definition of last quarter (e.g., calendar or fiscal), and the interpretation of schema-specific entities such as regions and locations.
We model the missing semantics as an underspecified intent variable $Z$ with finite support $\mathcal{Z}(x,S)$, the set of plausible user interpretations of $x$ under schema $S$. 
A realization $z \in \mathcal{Z}(x,S)$ denotes one concrete intent assignment.
Conditioned on $z$, an NL2SQL system induces a distribution over executable SQL queries:
\(
P_\theta(q\mid x,S,z), q\in\mathcal{Q}(S),
\)
where $\theta$ denotes the pipeline configuration of the NL2SQL system~\cite{pourreza2024,talaei2024chess}.
The core challenge is that the user's true intent $z^\star$ is unobserved, making it impossible to directly condition on $z^\star$ during generation.

\begin{problem*}
\label{prob:somasql}
Given a natural language question $x$, a database schema $S$, and a database instance $D$, the task is to produce an executable SQL query $\hat{q} \in \mathcal{Q}(S)$ whose execution semantics match those of the gold SQL query $q^\star$.
Since the user's true intent $z^\star \in \mathcal{Z}(x,S)$ is unobserved, this requires
(i) identifying the finite set of plausible intent assignments $\mathcal{Z}(x,S)$,
(ii) estimating $\hat{z} \in \mathcal{Z}(x,S)$ from execution-grounded evidence, and
(iii) generating the final SQL as $\hat{q} \in \arg\max_{q \in \mathcal{Q}(S)}, P_\theta(q \mid x, S, \hat{z})$.
\end{problem*}

\textbf{Ambiguity dimensions.}
We decompose the {underspecified intent variable }into \emph{ambiguity dimensions}
\(
Z=(Z_1,\ldots,Z_m),
\)
where each \(Z_j\) captures one unresolved semantic decision, such as the ranking metric, time convention, filter, aggregation, or join path. 
A realized intent assignment
\(
z=(z_1,\ldots,z_m),
\)
selects a concrete value $z_j$ for each dimension \(Z_j\). 
The dimensions need not be independent; valid joint assignments may form only a subset of the Cartesian product of the dimension option sets. 
For example, in ``{\it top-performing regions in New York last quarter},'' \(Z_1\) (ranking metric), \(Z_2\) (time convention), and \(Z_3\) (entity reference) each admit multiple options, resolving them as \(z_1=\text{\textit{revenue}}\), \(z_2=\text{\textit{fiscal quarter}}\), and \(z_3=\text{\textit{New York State}}\) yields a fully specified intent from which the target SQL is deterministically derived.
%
\subsection{Multi-SQL Generation Workflow}
\label{sec:method:workflow}

\sys is an ambiguity-first NL2SQL framework built around a widely adopted plan-generate-critique workflow, as illustrated in Fig.~\ref{fig:archi}. 
Before invoking the workflow, \sys enriches the database schema~\cite{XiYanSQL,chen2025enrichindex} and performs agentic schema linking~\cite{wang2025autolink} to identify the relevant subset of $S$ for each question $x$\footnote{For ease of notation, we refer to this subset as $S$ throughout.}. 
This step avoids prompt context overflow (see Appendix~\ref{appendix:additional-details} for details). 
Prior to intent resolution, the workflow induces a distribution
\(
P_\theta(q\mid x,S),
\)
reflecting uncertainty over the underspecified intent. 
Running the workflow $K$ times produces candidates
\(
q_k \sim P_\theta(q \mid x,S) \text{ for } k=1,\ldots,K,
\)
defining the diagnostic candidate set
\[
\mathcal{Q}_K(x,S)=\{q_1,\ldots,q_K\}\subset\mathcal{Q}(S).
\]
These candidates serve as diagnostic samples rather than final answers: disagreements in their choice of metric, filter, join path, or temporal constraint expose the unresolved ambiguity dimensions defined in Sec.~\ref{sec:method:problem}.
In each run, the planner produces a structured plan
\(
\pi_k = \mathrm{Plan}(x,S)
\)
specifying relevant tables, joins, filters, aggregations, grouping keys, ordering logic, and temporal constraints. 
The dialect-aware generator then translates the plan into an initial SQL query: 
\(
q_k=\mathrm{Gen}(x,S,\pi_k).
\)

The critique module refines $q_k$ using schema validation, abstract syntax tree (AST) parsing~\cite{sqlglot}, SQL planning diagnostics, and lightweight execution diagnostics:
\(
q_k=\mathrm{Critique}(x,S,D,\pi_k,q_k).
\) 
Following~\citet{gong2026sqlens} for parsing-based and execution diagnostics, we retain only the checks that provided the largest gains and generalized across datasets: invalid join predicates unsupported by the schema, suboptimal join trees containing unnecessary tables, incorrect \texttt{GROUP BY} structure, and abnormal execution outputs (e.g., all-zero results or \texttt{NaN}s). 
We encode these high-impact diagnostics directly into the critique prompt.
Repeating this loop for $k=1,\ldots,K$ yields $\mathcal{Q}_K(x,S)$. 
We optimize the planner's system prompt to mitigate common failure modes; prompts for each stage are in Appendices~\ref{app:plan-prompt}, \ref{app:sql-prompt}, \ref{app:compile-prompt}, and~\ref{app:critique-prompt}.

\sys resolves the unresolved ambiguity dimensions in $\mathcal{Q}_K(x,S)$ using two complementary sources of evidence, namely synthetic query log (Sec.~\ref{sec:method:log}) and execution-grounded probing (Sec.~\ref{sec:method:probing}), to produce an evidence-supported intent estimate $\hat{z}$, conditioned on which the final SQL is generated.

\subsection{Ambiguity Supervision from Synthetic Query Log}
\label{sec:method:log}

\paragraph{Ambiguity resolution set construction.} 
A SQL query log contains validated SQL queries executed over a schema, without associated NL questions. Such logs are common in enterprise databases; when available, \sys leverages them to construct reusable supervision for recurring ambiguities and failure patterns. Algorithm~\ref{alg:synthetic-log} summarizes the synthetic ambiguity supervision pipeline.

\textbf{Intent anchors.}
Each logged SQL query $q_i^{\log}$ serves as an \emph{intent anchor}, representing a historically resolved user need. 
For each anchor, \sys prompts an LLM to generate a fully specified NL description, and then derives several underspecified variants $x_i^{(j)}$ by removing details along common ambiguity dimensions~\cite{ding2026}. 
The detailed prompts are in Appendices~\ref{sec:nl-question-generation} and \ref{sec:ambiguous-nl-question-generation}.

\textbf{Candidate disagreement.}
For each $x_i^{(j)}$, \sys runs the multi-SQL generation workflow to produce $K$ candidate SQL queries, $\mathcal{Q}_K(x_i^{(j)}, S) \subset \mathcal{Q}(S)$. 
Because $x_i^{(j)}$ is intentionally underspecified, these candidates often reflect distinct plausible interpretations. 
\sys presents all $K$ candidates to an LLM and prompts it to identify implementation-level differences across the full SQL structure. 
The resulting differences are consolidated into a set of discrepancies $\Delta_i^{(j)}$, with each unique discrepancy recorded once. 
For example, two candidate filters, \textit{city = `NYC'} and \textit{city = `New York'}, are consolidated into a predicate-value discrepancy set as a single entry in $\Delta_i^{(j)}$. 

\textbf{Ambiguity supervision.}
\sys converts $\Delta_i^{(j)}$ into structured ambiguity records, 
each a triple $(Z_j, \mathcal{O}, \mathcal{E})$, where $Z_j$ is an ambiguity dimension, $\mathcal{O}$ is the set of plausible semantic options, and $\mathcal{E} \subseteq \Delta_i^{(j)}$ is the set of implementation-level SQL differences that induce the distinction. 
For example, if candidates disagree between \textit{city = `New York City'} and \textit{state = `New York'}, this induces an entity-reference ambiguity dimension with options \{{\it New York City, New York State}\} and evidence drawn from the observed predicate differences.

More generally, each option represents a distinct resolution of $Z$, such as a choice of aggregation, filter condition, grouping level, or join path. The option set $\mathcal{O}$ is induced from candidate disagreement, capturing both correct and plausible-but-incorrect interpretations. The anchor $q_i^{\log}$ identifies the historically supported option $o \in \mathcal{O}$; if absent, \sys adds a grounded option derived from $q_i^{\log}$. This converts unlabeled candidate disagreement into labeled ambiguity supervision. Each supervision item is stored as
\(
c = (x_i^{(j)}, S, Z, \mathcal{O}, o, \mathcal{E}),
\)
which isolates a semantic decision, enumerates competing resolutions, and links them to concrete SQL differences.

\textbf{Ambiguity resolution set.}
Aggregating all MCQ-style records across anchors and their variants yields the {\it ambiguity resolution set} $\mathcal{C}$. 
It provides dimension-level supervision, namely identifying which semantic decisions are typically underspecified, which alternatives are plausible, which resolution is historically supported, and what SQL evidence distinguishes the alternatives. 


{\bf Example.}
Starting from an intent anchor $q_i^{\log}$, \sys generates a detailed NL question, ``{\it List the top-5 regions by total revenue in the most recent fiscal quarter},'' and removes key details to obtain the underspecified question $x_i^{(j)}$: ``{\it top-5 performing regions last quarter}.'' 
The multi-SQL generation workflow then produces $K$ candidates that resolve the missing details differently: candidates may rank regions by \texttt{SUM(revenue)}, \texttt{SUM(profit)}, or \texttt{COUNT(deals)}, and may interpret ``last quarter'' as calendar rather than fiscal. 
From the resulting discrepancies $\Delta_i^{(j)}$, \sys extracts two ambiguity dimensions: $Z_1 =$ ``ranking metric,'' with $\mathcal{O}_1 = \{\text{revenue}, \text{profit}, \text{deal count}\}$, and $Z_2 =$ ``time convention,'' with $\mathcal{O}_2 = \{\text{calendar quarter}, \text{fiscal quarter}\}$. 
The intent anchor $q_i^{\log}$ selects $o_1 = \text{revenue}$ and $o_2 = \text{fiscal quarter}$, yielding two labeled MCQ items added to $\mathcal{C}$.

\paragraph{Applying the ambiguity resolution set at inference time.}
When $\mathcal{C}$ is available, \sys first retrieves $m'$ MCQ records most similar to the user question $x$ and schema $S$, and then uses a reranker (see Appendix~\ref{sec:ambiguity-reranker-prompt}) to select the most relevant $m$ examples:
\[
\mathcal{C}_x=\mathrm{Rerank}_{m}\!\left(x,S,\mathrm{Retrieve}_{m'}(x,S,\mathcal{C})\right), m<m'.
\]
The retrieved MCQs are provided as in-context examples to the planning, SQL generation, and critique prompts. 
They expose likely ambiguity dimensions, indicate historically supported resolutions, and flag wrong implementations that the model has previously found plausible. 
If historical log is unavailable, this step is omitted and $\mathcal{C}_x=\varnothing$.

\subsection{Ambiguity Discovery, Probing, and Final Generation}
\label{sec:method:probing}

At inference time, 
\sys combines ambiguity-set evidence with database probing to resolve ambiguities. 
Given a question $x$, schema $S$, and retrieved records $\mathcal{C}_x$ if available,
the workflow described in Sec.~\ref{sec:method:workflow} samples $K$ diagnostic candidates, $\mathcal{Q}_K(x,S)$. Algorithm~\ref{alg:ambiguity-probing} summarizes the full ambiguity probing and repair procedure.
\sys first selects a seed SQL,
\[
q^{(0)}=\mathrm{Select}(x,S,\mathcal{Q}_K(x,S)),
\]
using an LLM judge (Appendix \ref{app:sql-selection-judge-prompt}) to take $x$, $S$, and all $K$ sampled SQLs as input and selects the one SQL closest to the intent of $x$. 
The seed is not assumed to be correct; it is the single candidate subsequently repaired. 
\sys identifies ambiguity dimensions to probe from two complementary signals.
First, comparing SQL structures within $\mathcal{Q}_K(x,S)$ yields $m_c$ candidate-induced ambiguity dimensions 
\(
\{(Z_t^{\mathrm{cand}},\mathcal{E}_t^{\mathrm{cand}})\}_{t=1}^{m_c},
\)
where $\mathcal{E}_t^{\mathrm{cand}}$ contains the implementation-level differences, such as alternative aggregations, filters, grouping keys, or join paths, that surface the $t$-th ambiguity dimension $Z_t^{\mathrm{cand}}$. 
Because candidate differences alone can miss plausible interpretations or conflate distinct semantic choices into similar implementations, \sys additionally prompts an LLM with $(x,S)$ and a database-sourced ambiguity taxonomy~\citep{ding2026} to obtain $m_t$ taxonomy-induced dimensions
\(
\{Z_j^{\mathrm{tax}}\}_{j=1}^{m_t},
\)
where $Z_j^{\mathrm{tax}}$ is the $j$-th taxonomy-induced ambiguity dimension. 
The union of both sets defines the probing targets. 
\paragraph{Probing unresolved dimensions.}
Identifying an ambiguity dimension reveals what is underspecified, but not its intended resolution. For each dimension, \sys generates targeted probe SQLs and executes them on the database instance $D$. Probe SQLs are executable queries designed to gather evidence for a specific ambiguity dimension rather than answer the user query directly. Candidate-induced probes are guided by $\mathcal{E}_t^{\mathrm{cand}}$ and schema context, while taxonomy-induced probes instantiate plausible alternatives from schema elements and database values.


Consider the query \emph{``show me churn for enterprise customers last quarter.''} A discovered ambiguity dimension concerns the definition of \emph{churn}: it may refer to canceled subscriptions, inactive accounts, lost revenue, or customers with no recent purchases. \sys constructs probe SQLs for each interpretation and executes them against the database. If the canceled-subscription interpretation yields substantially stronger and more temporally consistent evidence than the alternatives, the probe results favor that interpretation. A full prompt-and-report example is provided in Appendix~\ref{app:probe-example}.


\paragraph{Grounding report and resolution.}
The probe outputs are compiled into a grounding report
\[
R=\mathrm{Probe}
\!\left(
x,S,D,
\{(Z_t^{\mathrm{cand}},\mathcal{E}_t^{\mathrm{cand}})\}_{t=1}^{m_c},
\{Z_j^{\mathrm{tax}}\}_{j=1}^{m_t}
\right).
\]
$R$ records, for each ambiguity dimension, the execution results and database-grounded evidence for each candidate alternative. For candidate-induced dimensions, the alternatives are drawn from $\mathcal{E}_t^{\mathrm{cand}}$; for taxonomy-induced dimensions, they are first instantiated from schema elements and database values before probing.
\sys then resolves the discovered dimensions with an LLM-based resolver:
\[
\hat z
=
\mathrm{Resolve}
\!\left(
x,S,\mathcal{C}_x,
\{(Z_t^{\mathrm{cand}},\mathcal{E}_t^{\mathrm{cand}})\}_{t=1}^{m_c},
\{Z_j^{\mathrm{tax}}\}_{j=1}^{m_t},
R
\right),
\]
producing an evidence-supported intent estimate $\hat z \in \mathcal{Z}(x,S)$. It denotes the resolved interpretations of the candidate-induced and taxonomy-induced ambiguity dimensions, based on the accumulated evidence in $R$. Ambiguity resolution prompt is in \ref{app:ambi-resolution-prompt}.
\paragraph{Final generation.}
Given $\hat z$, the seed SQL $q^{(0)}$, and the grounding report $R$, \sys generates the final SQL by sampling from the intent-conditioned distribution 
\(
\hat q
\sim
P_\theta
\bigl(q \mid x,S,\hat z,q^{(0)},R\bigr).
\)
Diagnostic candidates are used to expose ambiguity, probe SQLs provide database-grounded evidence in $R$, and the resolver produces the clarified intent $\hat z$. 
The final SQL is generated by repairing the seed query under this evidence-supported interpretation.

\subsection{Automated Disambiguation as Human Clarification Proxy}
\label{sec:theoretical-guarantee}


Given $x$, $S$, and $\mathcal{C}_x$ if available, \sys generates
diagnostic candidates $\mathcal{Q}_K(x,S)$, detects ambiguity dimensions
from candidate differences and taxonomy-guided signals, and probes the
database to collect resolution evidence. In an ideal human-in-the-loop
setting, detected ambiguities would be presented to the user as an MCQ,
and the response would directly reveal the intended interpretation. We formalize this ideal setting using three random variables: $Z$ denotes the latent user intent, as defined above; $U$ denotes the ideal human clarification signal induced by the MCQ interaction; and $Y$ denotes the automated evidence constructed by \sys:
\[
Y \triangleq
\Bigl(
  \mathcal{Q}_K(x,S),\;
  \mathcal{C}_x,\;
  \{(Z_t^{\mathrm{cand}},\mathcal{E}_t^{\mathrm{cand}})\}_{t=1}^{m_c},\;
  \{Z_j^{\mathrm{tax}}\}_{j=1}^{m_t},\;
  R
\Bigr),
\]
where $Y$ historical log based resolution, aggregates candidate SQLs, retrieved ambiguity records, detected ambiguity dimensions, and probe-execution results.

\begin{proposition}[Automated proxy for clarification]
\label{prop:somasql-ambiguity-reduction}
For $(x,S)$, suppose the automated proxy $Y$ is a noisy approximation of the ideal human clarification signal $U$ in the sense that, after conditioning on $U$ and $(x,S)$, observing $Y$ provides no further information about $Z$.
Then,
\(
I(Z;Y\mid x,S) \leq I(Z;U\mid x,S).
\)
Thus, human clarification is an information-theoretic ceiling, and \sys constructs a tractable automated proxy below this ceiling.
\end{proposition}

The proof is in Appendix~\ref{app:ambiguity-reduction-proof}. 
The proposition characterizes when \sys can approach the human-in-the-loop ceiling.
Real database workloads are repetitive: for a fixed schema $S$, many questions follow recurring intent patterns.
As historical log covers these patterns, retrieved ambiguity records $\mathcal{C}_x$ recover previously resolved ambiguity decisions, narrowing the gap between $I(Z;Y\mid x,S)$ and $I(Z;U\mid x,S)$. 
The remaining gap is concentrated on rare or novel intents where the log lacks precedent.
\section{Experiments}
\label{sec:experiments}

\subsection{Experimental Setup}
\label{sec:exp:setup}

\textbf{Datasets.}  
We evaluate \sys on six NL2SQL benchmarks: Spider 2.0 Lite~\cite{lei2024spider}, BIRD-dev~\cite{li2023bird}, Archer-dev (English and Chinese)~\cite{zheng2024archer}, BEAVER (Oracle dialect)~\cite{chen2024beaver,beaver_oracle_conversion}, and AMBROSIA~\cite{saparina2024ambrosia}. Details can be found in Appendix~\ref{appendix:datasets}.

\textbf{Evaluation metrics.}
Following prior work~\cite{deng2025reforce,li2025deepeye,wang2025autolink}, we report Execution Accuracy (EX) as the primary metric. A prediction is correct if the predicted SQL program executes successfully and its result matches that of the ground-truth SQL program under the benchmark's evaluator. We use each benchmark's official evaluator where available; otherwise, we fall back to the evaluation logic provided by Spider 2.0 Lite.

\textbf{Baselines.}
We compare \sys against AutoLink~\cite{wang2025autolink} and DeepEye-SQL~\cite{li2025deepeye}, the leading open-source methods on the Spider 2.0 Lite and BIRD leaderboards, respectively.



\paragraph{Models.}
We evaluate all methods using three backbone LLMs: GPT-5.3-Codex~\cite{openai_gpt53_codex}, GPT-5~\cite{openai_gpt5}, and Gemma-4-31B~\cite{gemma4}, representing coding-oriented, general-purpose, and open-source models.

\textbf{Implementation settings.}
We use GPT-5.3-Codex to generate synthetic query log and sentence-transformers/all-MiniLM-L6-v2 for retrieval embeddings; we rerank the top $10$ candidates with GPT-5.4~\cite{openai2026gpt54} and keep up to $3$. For ambiguity-driven probing, we use codex-cli (v0.111.0) and implement the probing workflow using agent skills (see Appendix~\ref{app:ambi-resolution-prompt}). For multi-SQL generation, we sample $10$ candidate SQLs and allow up to $5$ critique rounds. 

\begin{table}[t]
\caption{Execution accuracy (EX) in \% across benchmarks without synthetic query log; ``Avg.'' is the mean across Spider 2.0 Lite, BIRD, Archer-en, Archer-zh, and BEAVER. Values highlighted in blue and annotated with gains denote absolute improvements over the second-best result. \sys consistently outperforms baselines across datasets, dialects, languages, and backbone models.}
\label{tab:main_results}
\centering
\small
\setlength{\tabcolsep}{3pt}
\resizebox{\linewidth}{!}{
\begin{tabular}{l l c c c c c c}
\toprule
\textbf{LLMs} & \textbf{Methods} & \textbf{Spider 2.0 Lite} & \textbf{BIRD} & \textbf{Archer-en} & \textbf{Archer-zh} & \textbf{BEAVER} & \textbf{Avg.} \\
\midrule

\multirow{3}{*}{GPT-5.3 Codex}
& \textbf{\sys}
& \textbf{69.3}{\color{blue}(+23.0)}
& \textbf{73.1}{\color{blue}(+5.0)}
& \textbf{71.2}{\color{blue}(+2.9)}
& \textbf{78.9}{\color{blue}(+5.8)}
& \textbf{30.1}{\color{blue}(+7.6)}
& \textbf{64.5}{\color{blue}(+9.3)} \\

& AutoLink
& 46.3 & 65.6 & 68.3 & 73.1 & 22.5 & 55.2 \\

& DeepEye-SQL
& 33.6 & 68.1 & 41.4 & 50.0 & 12.0 & 41.0 \\

\midrule

\multirow{3}{*}{GPT-5}
& \textbf{\sys}
& \textbf{65.8}{\color{blue}(+20.6)}
& \textbf{71.0}{\color{blue}(+4.7)}
& \textbf{72.1}{\color{blue}(+7.7)}
& \textbf{84.6}{\color{blue}(+14.4)}
& \textbf{36.4}{\color{blue}(+13.4)}
& \textbf{66.0}{\color{blue}(+12.8)} \\

& AutoLink
& 45.2 & 63.3 & 64.4 & 70.2 & 23.0 & 53.2 \\

& DeepEye-SQL
& 31.4 & 66.3 & 49.0 & 49.0 & 14.1 & 42.0 \\

\midrule

\multirow{3}{*}{Gemma-4-31B}
& \textbf{\sys}
& \textbf{53.7}{\color{blue}(+16.0)}
& \textbf{69.7}{\color{blue}(+10.4)}
& \textbf{65.4}{\color{blue}(+9.6)}
& \textbf{75.0}{\color{blue}(+20.2)}
& \textbf{36.3}{\color{blue}(+21.0)}
& \textbf{60.0}{\color{blue}(+16.7)} \\

& AutoLink
& 37.7 & 52.9 & 55.8 & 54.8 & 15.3 & 43.3 \\

& DeepEye-SQL
& 24.0 & 59.3 & 43.3 & 42.3 & 8.5 & 35.5 \\

\bottomrule
\end{tabular}}
\end{table}

\subsection{Main Results}

\paragraph{Comparison with baselines.}
We compare \sys with AutoLink and DeepEye-SQL across five NL2SQL benchmarks:
Spider 2.0 Lite, BIRD, Archer-en, Archer-zh, and BEAVER.
To ensure a fair comparison, we assume \textit{no access to query log}
as additional disambiguation signals during inference, a setting consistent
across all methods. Table~\ref{tab:main_results} summarizes the results
across these benchmarks.
First, \sys achieves consistent gains across all benchmarks and
models, improving over the strongest baseline by 9.3\%,
12.8\%, and 16.7\% on average for GPT-5.3 Codex,
GPT-5, and Gemma-4-31B, respectively, demonstrating the generalizability
of our approach without reliance on any dataset-specific supervision.
Second, competing baselines fail to generalize across datasets.
AutoLink, submitted to Spider 2.0 leaderboard, ranks second compared to \sys but
is outperformed by DeepEye-SQL on BIRD. Conversely, DeepEye-SQL, which
is tuned for BIRD, underperforms AutoLink on every other benchmark.
This cross-dataset inconsistency reveals that prior methods overfit to
their target leaderboards, whereas \sys maintains top performance across
all five benchmarks.
Third, \sys generalizes across SQL dialects and languages.
Spider 2.0 and BEAVER together span Snowflake, BigQuery, SQLite, and
Oracle dialects, and \sys achieves the largest absolute gains on
these benchmarks, indicating robustness to dialect variation.
Furthermore, strong improvements on Archer-zh confirm cross-lingual
generalization.


\begin{table}
  \caption{Execution accuracy (EX) on subsampled AMBROSIA partitions with Gemma-4-31B and no synthetic query log. Gap is the drop from unambiguous to ambiguous EX.}
  \label{tab:ambrosia_gemma}
  \centering
  \small
  \begin{tabular}{lccc}
    \toprule
    \textbf{Methods} & \textbf{Ambiguous (EX)} & \textbf{Unambiguous (EX)} & \textbf{Gap} \\
    \midrule
    \sys & \textbf{92.0} & \textbf{94.7} & \textbf{2.7} \\
    AutoLink & 75.3 & 86.7 & 11.4 \\
    DeepEye-SQL & 64.7 & 72.0 & 7.3 \\
    \bottomrule
  \end{tabular}
\end{table}

\paragraph{Results on ambiguity-focused benchmark.}
To directly evaluate ambiguity resolution, we use AMBROSIA, an NL2SQL benchmark where each ambiguous question admits multiple valid interpretations. We evaluate \sys using Gemma-4-31B on a balanced 300-example split: 150 ambiguous questions and 150 unambiguous questions. The ambiguous set contains 50 examples from each ambiguity category (i.e., attachment, scope, and vague references)~\cite{saparina2024ambrosia}. Table~\ref{tab:ambrosia_gemma} shows that \sys achieves 92.0\% EX on ambiguous questions and 94.7\% EX on unambiguous questions, outperforming both DeepEye-SQL and AutoLink on both partitions. The ambiguous--unambiguous gap is also smallest for \sys: 2.7\%, compared to 11.4\% for AutoLink and 7.3\% for DeepEye-SQL. Thus, \sys retains most of its accuracy under ambiguous inputs, suggesting its robustness to semantic underspecification.

\begin{table}
  \caption{Effectiveness of synthetic query log. We report execution accuracy (EX).}
  \label{tab:ablation_logs}
  \centering
    \small
  \begin{tabular}{llcc}
    \toprule
    \textbf{Datasets} & \textbf{LLMs} & \textbf{w/o Log} & \textbf{w/ Log} \\
    \midrule
    \multirow{2}{*}{Spider 2.0} 
      & Gemma-4-31B    & 53.7 & \textbf{61.0} \\
      & GPT-5.3 Codex  & 69.3 & \textbf{72.0} \\
    \midrule
    \multirow{2}{*}{BEAVER}     
      & Gemma-4-31B    & 36.3 & \textbf{46.4} \\
      & GPT-5.3 Codex  & 30.1 & \textbf{38.8} \\
    \midrule
    \multirow{2}{*}{AMBROSIA (Ambiguous)}     
      & Gemma-4-31B    & 92.0 & \textbf{94.7} \\
      & GPT-5.3 Codex  & 89.3 & \textbf{93.3} \\
    \bottomrule
  \end{tabular}
\end{table}

\subsection{Impact of Synthetic Query Log}
\label{sec:ablation_logs}

We study the effectiveness of targeted synthetic query log on improving \sys beyond the no-log setting in Table~\ref{tab:main_results}. 
We choose Spider 2.0 Lite, BEAVER, and AMBROSIA to cover diverse SQL dialects, schema complexity, and ambiguous question types. 
For each dataset, we use the provided gold SQLs only as intent anchors for generating synthetic query log entries, following Sec.~\ref{sec:method:log}. 
Note, no gold SQL is exposed in the MCQs or added to the NL2SQL generation workflow. 
Instead, we derive ambiguity-focused MCQs, use them to disambiguate the synthetic question, and convert the selected option into an NL clarification stored in the synthetic query log. 
Examples are provided in Appendix~\ref{app:derived_mcq_examples}. 
Table~\ref{tab:ablation_logs} shows that the synthetic query log consistently improve execution accuracy across datasets and backbone models. 
On Spider 2.0 Lite, \sys with query log improves performance by 7.3\% and 2.7\% for Gemma-4-31B and GPT-5.3 Codex, respectively. On BEAVER, the gains are larger, with improvements of 10.1\% and 8.7\%, respectively. On AMBROSIA, query log further improve performance by 2.7\% and 4.0\%. These results show that targeted synthetic query log provide useful ambiguity grounding beyond schema information.

\textbf{Helpfulness of synthetic MCQ clarifications for SQL generation.}
Using an LLM-as-a-Judge~\cite{gu2024survey} described in Appendix~\ref{app:nl2sql-clarification-helpfulness-prompt}, we evaluate whether each synthetic MCQ clarification provides useful guidance for generating the intended SQL. 
Given the original question, the synthetic MCQ clarification, and the gold SQL if available, the judge assigns one of two labels:
\texttt{helpful} if the clarification guides the NL2SQL system toward the
intended SQL semantics, 
and \texttt{not\_helpful} if it provides little or no useful guidance for SQL generation. 
The judge labels 145/150 AMBROSIA examples (96.7\%), 207/209 BEAVER examples (99.0\%), and 421/456 Spider 2.0 Lite examples (92.3\%) as \texttt{helpful}. 
These results confirm that the synthetic MCQs generally provide useful guidance for generating the intended SQL.




\begin{table}[t]
  \caption{Effectiveness of ambiguity-aware SQL selection and probing across
datasets. ``Avg.'' is the mean across Spider 2.0 Lite, BEAVER, and AMBROSIA.}
  \label{tab:ablation_selector_probing}
  \centering
  \small
  \setlength{\tabcolsep}{3pt}
  \begin{tabular}{lcccc}
    \toprule
    \textbf{Method} & \textbf{Spider 2.0 Lite} & \textbf{BEAVER} & \textbf{AMBROSIA} & \textbf{Avg.} \\
    \midrule
    Majority Voting & 41.1 & 41.6 & 92.6 & 58.4 \\
    LLM-as-a-Judge & 46.1 & 43.1 & 92.0 & 60.4 \\
    LLM-as-a-Judge (Seed Selection) + Probing & 61.0 & 46.4 & 94.7 & 67.4 \\
    \bottomrule
  \end{tabular}
\end{table}

\subsection{Impact of Ambiguity-Driven Probing}
\begin{table}[t]
 \caption{Accuracy gains from ambiguity-driven probing on difficult Spider 2.0 Lite instances.}
  \label{tab:spider2_probing_gains_hard}
  \centering
  \small
  \setlength{\tabcolsep}{4pt}
  \begin{tabular}{lccc}
    \toprule
    \textbf{Bucket} & \textbf{w/o Probing} & \textbf{w/ Probing} & \textbf{Gain} \\
    \midrule
Sparsely correct (1--4/10) & 61.1 & 81.5 & +20.4 \\
Never correct (0/10) & 0.0 & 30.6 & +30.6 \\
    \bottomrule
  \end{tabular}
\end{table}

We evaluate probing on Spider 2.0, BEAVER, and AMBROSIA, covering diverse SQL dialects and ambiguous question types. All methods are evaluated in the query-log setting described in Sec~\ref{sec:ablation_logs}. We compare our full method, LLM-as-a-Judge (LLMaaJ)-based seed selection with ambiguity-driven probing, against two baselines: Majority Voting and naive LLMaaJ. The LLMaaJ prompt is provided in Appendix~\ref{app:sql-selection-judge-prompt}. As shown in Table~\ref{tab:ablation_selector_probing}, our method, on average, improves performance by 9.0\% over Majority Voting (67.4 vs.\ 58.4) and 7.0\% over LLMaaJ (67.4 vs.\ 60.4), demonstrating the effectiveness of ambiguity-driven probing beyond selection alone.

\paragraph{Probing primarily benefits hard instances.} 
To understand the impact of ambiguity-driven probing, we characterize difficulty on Spider 2.0-Lite by running \sys with Gemma-4-31B ten times per instance and recording the number of correct executions. 
We label instances with zero correct runs as \textit{never correct} and those with one to four correct runs as \textit{sparsely correct}; see Appendix~\ref{app:bucket-analysis} for details. 
Our analysis focuses on these hard cases, where the system succeeds in at most 4 out of 10 runs.
As shown in Table~\ref{tab:spider2_probing_gains_hard}, adding probing to LLMaaJ-based selection yields the largest improvements on the hardest instances. For \textit{never correct} instances, probing improves execution accuracy from 0\% to 30.6\%, recovering 82 previously unsolved cases. For \textit{sparsely correct} instances, probing improves execution accuracy from 61.1\% to 81.5\%. These results suggest that probing is particularly effective for difficult questions that selection alone cannot resolve.

\textbf{Quality of ambiguity probing.}
We evaluate ambiguity probing quality using LLMaaJ on Spider~2.0-Lite outputs generated by \sys with Gemma-4-31B. We measure three binary metrics: \textit{Probing Groundedness} (whether probes are justified by the question and grounded in implementation differences or ambiguity taxonomy evidence), \textit{Resolution Correctness} (whether the selected resolution matches user intent), and \textit{SQL Repair Faithfulness} (whether the repaired SQL reflects the chosen resolution without unrelated changes). We audit 50 proportionally sampled instances spanning cases where probing improves, preserves, or fails to change execution outcomes (e.g., wrong$\rightarrow$correct and wrong$\rightarrow$wrong). Details are provided in Appendix~\ref{app:probe-audit}. On the audited set, ambiguity probing achieves 0.94 probing groundedness, 0.90 resolution correctness, and 0.96 SQL repair faithfulness (0--1 scale), indicating that probes are generally well-grounded and downstream SQL repairs are highly faithful.

\vspace{-5pt}
\section{Conclusion}
\label{sec:conclusion}
\vspace{-5pt}

We present \sys, a novel ambiguity-aware framework for autonomous ambiguity resolution in text-to-SQL. \sys leverages synthetic query log, multi-SQL disagreement signals, and database-grounded probing to identify and resolve unresolved user intent. It detects ambiguity through structured reasoning over divergent SQL candidates, maps conflicts into interpretable ambiguity dimensions, and applies evidence-driven refinement to repair or select the final SQL query. Experiments on six public benchmarks demonstrate that \sys improves execution accuracy by 13.0\% on average over state-of-the-art baselines, with gains of up to 16.7\% on ambiguous questions, without requiring human interaction at inference time.
\section{Limitations}
\label{sec:limitations}
\vspace{-5pt}

\sys relies on LLM-based reasoning for ambiguity detection, interpretation, and SQL refinement, and therefore inherits limitations of the underlying models. In particular, ambiguity resolution depends on the diversity and semantic coverage of generated SQL candidates: when all candidates converge to the same incorrect interpretation, unresolved intent may remain undetected. Some ambiguity dimensions also require world knowledge or long-range reasoning beyond current frontier LLM capabilities, leading to incomplete or incorrect resolutions.

\bibliographystyle{plainnat}
\bibliography{somasql}

\nolinenumbers
\appendix

\appendix
\newpage

\section{Related Works}
\label{app:more_related}

{\bf Beyond NL2SQL: Ambiguity in NL.}
Ambiguity is a well-studied challenge in NLP broadly. AmbigQA~\cite{min2020ambigqa} establishes that open-domain questions frequently admit multiple valid interpretations. Rao and Daumé~\cite{rao2018learning} frame clarification question generation as maximizing expected information gain, while Aliannejadi et al.~\cite{aliannejadi2019asking} demonstrate the value of proactive clarification in conversational search. Keyvan and Huang~\cite{keyvan2022survey} survey ambiguity handling in task-oriented dialogue. Unlike these approaches, our work exploits database-intrinsic signals (e.g., SQL syntax and schema metadata) that are uniquely available in the NL2SQL setting to derive rich signals for disambiguation.

\section{Notation}
\label{sec:math-notation}

Table~\ref{tab:notation} summarizes the main notation used throughout the paper. 

\begin{table}[h]
\centering
\caption{Notations and operators.}
\small
\begin{tabular}{ll}
\toprule
\textbf{Notation} & \textbf{Meaning} \\
\midrule
$x, S, D$ & Natural-language question, database schema, and database instance. \\
$q, q^\star, \hat q$ & Candidate SQL, gold SQL, and final generated SQL program. \\
$\mathcal{Q}(S)$ & Set of valid executable SQL programs over schema $S$. \\
$\theta$ & Configuration of the underlying NL2SQL pipeline. \\
$P_\theta(q \mid \cdot)$ & SQL generation distribution under pipeline configuration $\theta$. \\
\midrule
$Z$ & Random variable representing the user's latent intent. \\
$z^\star, \hat{z}$ & The true intent for the given question-schema pair, and the intent estimated by \\
 & \sys, respectively. \\
$\mathcal{Z}(x,S)$ & Set of plausible intent assignments for $(x,S)$. \\
$H(Z \mid x,S)$ & Entropy measuring residual ambiguity in the question. \\
\midrule
$\mathrm{Plan}, \mathrm{Gen}, \mathrm{Critique}$ & Planning, dialect-aware SQL generation, and critique modules. \\
$\pi_k, q_k$ & Plan and generated SQL in the $k$-th run. \\
$\mathcal{Q}_K(x,S)$ & Diagnostic set of $K$ sampled SQL candidates. \\
$q^{(0)}$ & Seed SQL selected for probe-grounded repair. \\
$\mathrm{Select}$ & LLM-based procedure for selecting the seed SQL from $\mathcal{Q}_K(x,S)$. \\
\midrule
$\mathcal{L}, \mathcal{C}, \mathcal{C}_x$ & Query log, constructed ambiguity resolution set, and retrieved relevant ambiguity\\
 & resolution records. \\
$\mathrm{Retrieve}, \mathrm{Rerank}$ & Retrieval and reranking operators used to obtain $\mathcal{C}_x$. \\
$\Delta$ & Implementation-level SQL differences among diagnostic candidates. \\
$(Z,\mathcal{O},\mathcal{E})$ & Ambiguity record: dimension, options, and supporting SQL evidence. \\
\midrule
$Z^{\mathrm{cand}}, Z^{\mathrm{tax}}$ & Candidate-induced and taxonomy-induced ambiguity dimensions. \\
$R$ & Grounding report from executed probe SQLs. \\
$\mathrm{Probe}$ & Procedure that generates and executes diagnostic SQL probes. \\
$\mathrm{Resolve}$ & Resolver that estimates $\hat z$ from ambiguity evidence and probe results. \\
\midrule
$Y$ & Automated clarification signal constructed by \sys. \\
$U$ & Ideal human clarification signal obtained from the MCQ interaction. \\
\bottomrule
\end{tabular}
\label{tab:notation}
\end{table}

\section{Algorithm - Synthetic Query Log}
\begin{algorithm}[H]
\caption{\sys: Ambiguity Supervision from Synthetic SQL Query Log}
\label{alg:synthetic-log}
\begin{algorithmic}[1]
\Require SQL query log $\mathcal{L} = \{q_i^{\log}\}_{i=1}^N$, schema $S$, number of candidates $K$
\Ensure Ambiguity resolution set $\mathcal{C}$
\State Initialize ambiguity resolution set:
\State \hspace{1em} $\mathcal{C} \gets \varnothing$

\For{each intent anchor $q_i^{\log} \in \mathcal{L}$}
    \State \textbf{Intent Anchor Generation:}
    \State Generate fully specified NL question:
    \State \hspace{1em} $x_i \gets \mathrm{Describe}(q_i^{\log}, S)$
    
    \State Generate underspecified variants by removing details along common ambiguity dimensions:
    \State \hspace{1em} $\{x_i^{(j)}\}_{j=1}^{m_i} \gets \mathrm{Underspecify}(x_i)$

    \For{each underspecified question $x_i^{(j)}$}

        \State \textbf{Candidate Disagreement Extraction:}
        \State Generate candidate SQL queries:
        \State \hspace{1em} $\mathcal{Q}_K(x_i^{(j)}, S) \subset \mathcal{Q}(S)$
        
        \State Present all $K$ candidates to an LLM and extract implementation-level discrepancies:
        \State \hspace{1em} $\Delta_i^{(j)} \gets \mathrm{ConsolidateUniqueDiffs}(\mathcal{Q}_K(x_i^{(j)}, S))$

        \State \textbf{Ambiguity Supervision Construction:}
        \State Convert discrepancies into structured ambiguity records:
        \State \hspace{1em} $\{(Z,\mathcal{O},\mathcal{E})\} \gets \mathrm{ExtractAmbiguities}(\Delta_i^{(j)})$

        \For{each ambiguity record $(Z,\mathcal{O},\mathcal{E})$}

            \State Identify anchor-supported option:
            \State \hspace{1em} $o^* \gets \mathrm{MatchAnchor}(q_i^{\log}, \mathcal{O})$

            \If{$o^* \notin \mathcal{O}$}
                \State Add grounded option derived from $q_i^{\log}$ to $\mathcal{O}$
            \EndIf

            \State Construct MCQ-style supervision item:
            \State \hspace{1em} $c \gets (x_i^{(j)}, S, Z, \mathcal{O}, o^*, \mathcal{E})$

            \State Add supervision item to ambiguity resolution set:
            \State \hspace{1em} $\mathcal{C} \gets \mathcal{C} \cup \{c\}$

        \EndFor
    \EndFor
\EndFor

\State \Return $\mathcal{C}$

\end{algorithmic}
\end{algorithm}

\section{Algorithm - Ambiguity Discovery, Probing, and Repair}
\label{pseudo-code-probing}

\begin{algorithm}[H]
\caption{\sys: Ambiguity Discovery, Probing, and Repair}
\label{alg:ambiguity-probing}
\begin{algorithmic}[1]

\Require Question $x$, schema $S$, retrieved ambiguity records $\mathcal{C}_x$ (optional), database $D$, number of diagnostic candidates $K$
\Ensure Final SQL query $\hat q$

\State \textbf{Diagnostic Candidate Generation:}
\State Sample diagnostic SQL candidates:
\State \hspace{1em} $\mathcal{Q}_K(x,S)$

\State Select seed SQL using an LLM judge:
\State \hspace{1em} $q^{(0)} \gets \mathrm{Select}(x,S,\mathcal{Q}_K(x,S))$

\vspace{0.3em}
\State \textbf{Ambiguity Discovery:}

\State Extract candidate-induced ambiguity dimensions and implementation evidence:
\State \hspace{1em} $\{(Z_t^{\mathrm{cand}},\mathcal{E}_t^{\mathrm{cand}})\}_{t=1}^{m_c}$
\State \hspace{2em} $\gets \mathrm{ExtractCandidateAmbiguities}(\mathcal{Q}_K(x,S))$

\State Extract taxonomy-induced ambiguity dimensions:
\State \hspace{1em} $\{Z_j^{\mathrm{tax}}\}_{j=1}^{m_t}$
\State \hspace{2em} $\gets \mathrm{ExtractTaxonomyAmbiguities}(x,S)$

\State Form probing targets:
\State \hspace{1em}
$\mathcal{Z}
\gets
\{(Z_t^{\mathrm{cand}},\mathcal{E}_t^{\mathrm{cand}})\}_{t=1}^{m_c}
\cup
\{Z_j^{\mathrm{tax}}\}_{j=1}^{m_t}$

\vspace{0.3em}
\State \textbf{Database Probing:}

\For{each candidate-induced dimension $(Z_t^{\mathrm{cand}},\mathcal{E}_t^{\mathrm{cand}})$}

    \State Construct probe SQLs using implementation evidence:
    \State \hspace{1em}
    $\mathcal{P}_t
    \gets
    \mathrm{BuildCandidateProbes}
    (Z_t^{\mathrm{cand}},\mathcal{E}_t^{\mathrm{cand}},x,S,D)$

    \State Execute probes and collect grounding evidence:
    \State \hspace{1em}
    $R_t^{\mathrm{cand}}
    \gets
    \mathrm{Execute}(\mathcal{P}_t,D)$

\EndFor

\For{each taxonomy-induced dimension $Z_j^{\mathrm{tax}}$}

    \State Instantiate plausible alternatives from schema and database:
    \State \hspace{1em}
    $\mathcal{A}_j
    \gets
    \mathrm{InstantiateAlternatives}(Z_j^{\mathrm{tax}},S,D)$

    \State Construct probe SQLs:
    \State \hspace{1em}
    $\mathcal{P}_j
    \gets
    \mathrm{BuildTaxonomyProbes}
    (Z_j^{\mathrm{tax}},\mathcal{A}_j,x,S,D)$

    \State Execute probes and collect grounding evidence:
    \State \hspace{1em}
    $R_j^{\mathrm{tax}}
    \gets
    \mathrm{Execute}(\mathcal{P}_j,D)$

\EndFor

\State Aggregate grounding report:
\State \hspace{1em}
$R
\gets
\{
R_t^{\mathrm{cand}}
\}_{t=1}^{m_c}
\cup
\{
R_j^{\mathrm{tax}}
\}_{j=1}^{m_t}$

\vspace{0.3em}
\State \textbf{Ambiguity Resolution:}

\State Resolve ambiguity dimensions using retrieved records and grounding evidence:
\State \hspace{1em}
$\hat z
\gets
\mathrm{Resolve}
\!\left(
x,S,\mathcal{C}_x,
\{(Z_t^{\mathrm{cand}},\mathcal{E}_t^{\mathrm{cand}})\}_{t=1}^{m_c},
\{Z_j^{\mathrm{tax}}\}_{j=1}^{m_t},
R
\right)$

\vspace{0.3em}
\State \textbf{SQL Repair and Final Generation:}

\State Repair the seed SQL under resolved intent:
\State \hspace{1em}
$\hat q
\sim
P_\theta(q \mid x,S,\hat z,q^{(0)},R)$

\State \Return $\hat q$

\end{algorithmic}
\end{algorithm}

\section{Proof and Discussion of Proposition~\ref{prop:somasql-ambiguity-reduction}}
\label{app:ambiguity-reduction-proof}

\begin{proof}
Given a question-schema pair $(x,S)$, denote $Z$ as the latent user intent, $U$ the ideal human clarification signal induced by the MCQ interaction, and $Y$ the automated clarification proxy constructed by \sys.

We assume that, conditional on $(x,S)$, $Y$ is a noisy proxy for $U$, formalized by the Markov chain
\[
Z - U - Y \mid (x,S).
\]
Equivalently, once the ideal clarification signal $U$ is known, the proxy $Y$ provides no additional information about the latent intent:
\[
Z \perp\!\!\!\perp Y \mid U,x,S .
\]

By the conditional data-processing inequality~\citep{reza1994introduction},
\[
I(Z;Y\mid x,S) \leq I(Z;U\mid x,S).
\]
Thus, the information that \sys's automated proxy carries about the latent intent is upper-bounded by the information available from ideal human clarification.

\end{proof}

\paragraph{Why the Markov assumption is appropriate.}
The assumption does not mean that \sys observes $U$ directly.
Rather, it models $U$ as the ideal ambiguity-resolving signal: if a user were asked the MCQ produced from the discovered ambiguity options, their response would directly specify the intended resolution.
The automated proxy $Y$ is designed to approximate this clarification signal using only system-generated evidence:
\[
Y \triangleq
\Bigl(
  \mathcal{Q}_K(x,S),\;
  \mathcal{C}_x,\;
  \{(Z_t^{\mathrm{cand}},\mathcal{E}_t^{\mathrm{cand}})\}_{t=1}^{m_c},\;
  \{Z_j^{\mathrm{tax}}\}_{j=1}^{m_t},\;
  R
\Bigr).
\]
Here, $\mathcal{Q}_K(x,S)$ is the diagnostic candidate set, $\mathcal{C}_x$ is the retrieved confusion set,
$\{(Z_t^{\mathrm{cand}},\mathcal{E}_t^{\mathrm{cand}})\}_{t=1}^{m_c}$ are candidate-induced ambiguity dimensions with SQL evidence,
$\{Z_j^{\mathrm{tax}}\}_{j=1}^{m_t}$ are taxonomy-induced ambiguity dimensions, and $R$ is the probe-execution grounding report.
Together, these components constitute the evidence \sys uses to approximate the ambiguity resolution that would otherwise be supplied by the user.

\paragraph{When the proxy approaches the human ceiling.}
The gap
\[
I(Z;U\mid x,S)-I(Z;Y\mid x,S)
\]
measures the remaining value of explicit human clarification over \sys's automated proxy.
This gap is small when the historical log contains resolved examples of the same or similar ambiguity pattern.
In that case, retrieved confusion records act as synthetic human feedback: they identify ambiguity dimensions and historically supported resolutions that approximate the clarification a user would provide.

This is especially relevant for database workloads, which are often repetitive for a fixed schema.
As historical log accumulates coverage of recurring intent patterns, $\mathcal{C}_x$ becomes more likely to retrieve useful precedents.
For such covered patterns, $Y$ can approach the information content of $U$.
For rare or novel intents, however, the log lacks precedent, and the residual gap remains large.
These are precisely the cases where explicit human clarification is most valuable.

\paragraph{Uncertainty-reduction interpretation.}
Although the main proposition focuses on the human-clarification ceiling, the same evidence variable $Y$ also admits a standard uncertainty-reduction interpretation.
We quantify the ambiguity of a question by the conditional entropy $H(Z\mid x,S)$ of the underspecified intent $Z$ given $x$ and $S$.
High entropy indicates that many interpretations remain plausible after observing the question and schema, whereas low entropy indicates that only a few interpretations are likely.

By the definition of conditional mutual information,
\[
I(Z;Y\mid x,S)
=
H(Z\mid x,S)
-
\mathbb{E}_{Y\mid x,S}
\left[
H(Z\mid x,S,Y)
\right].
\]

Rearranging the expression gives
\[
\mathbb{E}_{Y\mid x,S}
\left[
H(Z\mid x,S,Y)
\right]
=
H(Z\mid x,S)
-
I(Z;Y\mid x,S).
\]

Therefore, if
\[
I(Z;Y\mid x,S) > 0,
\]
then
\[
\mathbb{E}_{Y\mid x,S}
\left[
H(Z\mid x,S,Y)
\right]
<
H(Z\mid x,S).
\]

Hence, conditioning on the evidence $Y$ produced by \sys decreases the expected posterior entropy of the latent intent $Z$. Equivalently, the observation $Y$ provides information about $Z$, reducing uncertainty in expectation.

This identity clarifies the contribution of each component of $Y$.
Diagnostic candidates $\mathcal{Q}_K(x,S)$ expose ambiguity through implementation-level disagreements.
Taxonomy-guided detection adds ambiguity dimensions that candidate disagreement may miss.
Confusion-set retrieval $\mathcal{C}_x$ supplies historically resolved ambiguity decisions when relevant log is available.
Probe execution contributes database-grounded evidence about the consequences of candidate resolutions.
Each component can increase the information that $Y$ carries about $Z$ beyond what is available from the question and schema alone.

\section{Dataset Details}
\label{appendix:datasets}

Spider 2.0 Lite contains 547 questions over 158 large, heterogeneous databases with up to 800+ columns and multiple SQL dialects (e.g., SQLite, Snowflake, BigQuery), emphasizing multi-step reasoning and dialect robustness. BIRD-dev contains 1,534 questions over 11 real-world databases with complex queries, realistic schemas, and noisy values. Archer-dev evaluates reasoning-intensive NL2SQL in English and Chinese, with 104 questions over 2 databases per split involving arithmetic, hypothetical, and commonsense reasoning. BEAVER (Oracle dialect) contains 209 enterprise-oriented questions over 6 databases with business-specific semantics and substantial schema and value ambiguity; we use the Oracle SQL version~\cite{beaver_oracle_conversion} to evaluate dialect robustness. Ambrosia~\cite{saparina2024ambrosia} evaluates robustness to semantic ambiguity with 4,242 questions containing multiple valid interpretations and corresponding SQL programs.

\section{Additional Details about \sys}
\label{appendix:additional-details}

Realistic databases often contain hundreds of tables and thousands of columns. Including the full schema in the prompt is therefore impractical: it increases context length and can cause prompt context overflow. Before running the main workflow, \sys first enriches the raw schema with additional semantic information, such as column descriptions, value hints, and index-derived evidence~\cite{XiYanSQL,chen2025enrichindex}. It then applies agentic schema~\cite{wang2025autolink} to select the subset of tables, columns, and relationships most relevant to the input question $x$. The downstream workflow operates only on this linked schema subset, which we denote by $S$ for notational simplicity. This preprocessing step keeps prompts compact while preserving the schema evidence needed for SQL generation.

We also tune the system prompt used by the planner. Following the prompt-refinement strategy of~\citet{liu2025oraplan}, we run a pilot study on a small held-out subset of the Archer training data~\cite{zheng2024archer}, inspect recurring failure modes, and incorporate the resulting corrective guidelines into the planner's system prompt. For dialect-specific instructions for SQLite, Snowflake, and BigQuery SQL, we follow~\citet{wang2025autolink} and explicitly condition SQL generation on the target database dialect. For Oracle dialect instructions, we follow~\cite{oracle_db_skills_sql_dev}.

\subsection{Prompt for SQL Plan Generation}
\label{app:plan-prompt}

\begin{promptbox}{SQL Plan Generation Prompt}

\textbf{Role}

You are a Plan Decomposer. Your task is to read a natural language question, a SQL database schema, and optional additional information. Then output a short, step-by-step plan in natural language that MUST be followed to write the SQL query.

\vspace{0.5em}

\textbf{Inputs}
\begin{itemize}
\item Question: \{question\}
\item Additional Information (Optional): \{external\_info\}
\end{itemize}

\vspace{0.5em}

\textbf{Output Format (Strict)}
\begin{enumerate}
\item ...
\item ...
\item ...
\end{enumerate}

Return columns: \textless ...\textgreater

\vspace{0.5em}

\textbf{General Rules}
\begin{itemize}
\item Use ONLY column and table names in schema. Do not invent names.
\item Prefer scalar subqueries for single values; use JOINs only when necessary.
\item Keep each step focused (3--6 steps total).
\item End with exactly one line:
\texttt{Return columns: \textless exact list\textgreater}
\item Use at most two guardrails per plan.
\end{itemize}

\vspace{0.5em}

\textbf{Projection Minimality (Strict)}
\begin{itemize}
\item Return only explicitly requested fields, in exact order.
\item ``Which \textless entity\textgreater\ ... ?'' $\rightarrow$ return only entity name.
\item Single-value questions $\rightarrow$ return a single scalar column.
\item ``Names and ages'' $\rightarrow$ exactly two columns: name, age.
\item Do not include helper or intermediate values unless requested.
\end{itemize}

\vspace{0.5em}

\textbf{Schema}

\begin{verbatim}
{schema}
\end{verbatim}

\end{promptbox}

\subsection{Prompt for SQL Generation}
\label{app:sql-prompt}

\begin{promptbox}{Prompt for SQL Generation}
You are an expert NL2SQL translator. Your task is to convert natural language questions into accurate and executable SQL queries for \{dialect\} databases.

\vspace{0.4em}

\textbf{Requirements}
\begin{itemize}
\item Interpret complex logic (e.g., relative time, comparisons, conditional joins).
\item Use correct SQL syntax and valid table relationships.
\item Use only valid \{dialect\} syntax.
\item Prefer UPPERCASE SQL keywords.
\item Output only raw SQL (no comments, no explanations, no markdown).
\end{itemize}

\vspace{0.4em}

\textbf{Inputs}
\begin{itemize}
\item Question: \{question\}
\item Execution Plan: \{plan\}
\item Additional Information: \{external\_info\}
\item Database Schema: \{schema\_txt\}
\end{itemize}

\vspace{0.4em}

\textbf{Output}
\begin{itemize}
\item A single valid SQL query.
\item No comments, no explanations, no extra text.
\end{itemize}

\vspace{0.4em}

\textbf{SQL Query}

\end{promptbox}

\subsection{Prompt for SQL Query Plan Correction}
\label{app:compile-prompt}

\begin{promptbox}{SQL Query Plan Correction Prompt}

\textbf{Role}

You are an experienced and professional database administrator tasked with analyzing and correcting SQL queries that are potentially wrong.

\vspace{0.5em}

\textbf{Context}

You are provided with:
\begin{itemize}
\item A \{dialect\} database schema
\item A user question
\item A proposed SQL query intended to answer the question
\item An error report for the proposed SQL query
\end{itemize}

\vspace{0.5em}

\textbf{Database Schema}
\begin{itemize}
\item The schema consists of table descriptions.
\item Each table contains multiple column descriptions.
\item Example values for each column are provided.
\end{itemize}

\vspace{0.5em}

\textbf{Instructions}
\begin{itemize}
\item Carefully review the provided context.
\item Analyze the error report.
\item If the SQL is incorrect, generate a corrected SQL query that answers the question.
\item Use only valid \{dialect\} syntax.
\item Prefer UPPERCASE SQL keywords.
\item Avoid unnecessarily complex queries.
\item Prefer JOINs over overly nested \texttt{EXISTS} clauses when equivalent.
\item Output only the SQL query.
\item Do not include explanations or comments.
\end{itemize}

\vspace{0.5em}

\textbf{Inputs}
\begin{itemize}
\item Question: \{question\}
\item External Knowledge: \{external\_info\}
\item Previous Compile Error Log: \{previous\_compile\_error\_log\}
\item Previous SQL: \{previous\_sql\}
\item Relevant Function Documentation: \{function\_documentation\}
\item Database Schema: \{schema\_txt\}
\end{itemize}

\vspace{0.5em}

\textbf{Task}

Analyze the error report and generate a corrected SQL query that answers the user question.

\vspace{0.5em}

\textbf{Output}

\begin{verbatim}
<correct SQL query>
\end{verbatim}

\end{promptbox}

\subsection{Prompt for SQL Critique}
\label{app:critique-prompt}

\begin{promptbox}{}

You are an experienced database administrator tasked with validating SQL and correcting it only when there is clear evidence of an error.

\vspace{0.25em}

\textbf{Inputs}
\begin{itemize}
\item A \{dialect\} database schema
\item A user question
\item A proposed SQL query
\item An error report, which may contain false positives
\end{itemize}

\vspace{0.25em}

\textbf{Database Schema}
\begin{itemize}
\item Table descriptions
\item Column descriptions
\item Frequent column values
\end{itemize}

\vspace{0.25em}

\textbf{Task}
\begin{itemize}
\item Analyze the error report.
\item Keep the SQL unchanged if it is already correct.
\item Otherwise, generate a corrected SQL query.
\end{itemize}

\vspace{0.25em}

\textbf{Editing Policy}
\begin{itemize}
\item If the error report is UNKNOWN, N/A, FALSE, or only warnings, return the original SQL unchanged.
\item Do not rewrite for style, readability, or optimization.
\item Edit only when there is a concrete, localizable defect.
\item If unsure, output the original SQL verbatim.
\end{itemize}

\vspace{0.25em}

\textbf{Instructions}
\begin{itemize}
\item Use valid \{dialect\} syntax only.
\item Output only SQL in a code block (no comments or explanations).
\item Avoid unnecessary complexity (e.g., prefer \texttt{JOIN} over \texttt{EXISTS} when equivalent).
\end{itemize}

\vspace{0.25em}

\textbf{Dialect Rules}

\{dialect\_rules\}

\vspace{0.25em}

\textbf{Inputs}
\begin{itemize}
\item Question: \{question\}
\item Old SQL:
\end{itemize}

\begin{quote}
\ttfamily
\{prev\_sql\}
\end{quote}

\begin{itemize}
\item Error Report:
\end{itemize}

\begin{quote}
\ttfamily
\{error\_report\}
\end{quote}

\begin{itemize}
\item Database Info: \{schema\_txt\}
\end{itemize}

\vspace{0.25em}

Determine whether the SQL requires correction and, if so, generate the corrected query.

\vspace{0.25em}

\textbf{Final SQL}

\end{promptbox}

\subsection{Prompt for Ambiguity Driven SQL Probing and Repair}
\label{app:probe-repair-prompt}

\begin{promptbox}{Prompt for Ambiguity Probing SQL Repair}

You are using \texttt{\$sql-error-fixing-core} skill to repair \{instance\_id\} with ambiguity driven SQL probing and repair. Your goal is to repair or select the SQL only when probe evidence supports a concrete correction or selection.

\vspace{0.25em}

\textbf{Instance Context}
\begin{itemize}
\item Dataset profile: \{dataset\_name\}
\item Current SQL: \{before\_fix\_sql\_path\}
\item Seed run name: \{seed\_run\_name\}
\item Seed source run for current SQL: \{seed\_source\_run\_name\}
\item Question text available: \{question\_text\_available\}
\item Question source: \{question\_text\_source\}
\item Question: \{question\_text\}
\item Seed mode: \{seed\_mode\}
\end{itemize}

\vspace{0.25em}

\textbf{Candidate SQL Context}
\begin{itemize}
\item Candidate SQL count from run\_records: \{candidate\_sql\_count\}
\item Candidate run names: \{candidate\_run\_names\}
\item Candidate SQL artifacts dir: \{candidate\_sql\_dir\}
\item Candidate SQL manifest path: \{candidate\_sql\_manifest\_path\}
\item Candidate SQL context: \{candidate\_sql\_overview\}
\end{itemize}

\vspace{0.25em}

\textbf{Probing Policy}
\begin{itemize}
\item Save all probing outputs under \{probe\_reports\_instance\_dir\}.
\item Use \{probe\_reports\_root\} as the probing reports root for this run.
\item If implementation diff is available, create 8--9 clarification probes total with exactly 3 implementation-diff-grounded probes and 5--6 baseline taxonomy-grounded probes.
\item If implementation diff is unavailable, create exactly 9 baseline taxonomy-grounded probes and 0 implementation-diff-grounded probes.
\item Treat implementation diff only as ambiguity hypotheses.
\item Do not patch directly from diff; patch only from probe evidence.
\end{itemize}

\vspace{0.25em}

\textbf{Scope and System Prompts}
\begin{itemize}
\item Scan scope policy: \{scan\_scope\_instructions\}
\item Read scope instructions: \{read\_scope\_instructions\}
\item SQL fixing system prompt: \{sql\_fixing\_system\_prompt\}
\item Probing resolution system prompt: \{probing\_resolution\_system\_prompt\}
\item System prompt files: sql\_fixing=\{sql\_fixing\_system\_prompt\_path\}, probing\_resolution=\{probing\_resolution\_system\_prompt\_path\}
\end{itemize}

\vspace{0.25em}

\textbf{Repair Target}
\begin{itemize}
\item Write the final repaired SQL to \{fixed\_sql\_path\}.
\item Save results under this exact root directory only: \{save\_results\_dir\}.
\end{itemize}

\vspace{0.25em}

Now perform ambiguity probing, resolve the SQL only from probe evidence, and produce the repaired SQL.

\vspace{0.25em}

\textbf{Final SQL}

\end{promptbox}

\subsection{Ambiguity-Driven SQL Probing and Repair Contract}
\label{app:ambi-resolution-prompt}

We implement ambiguity-driven SQL probing and repair with the \texttt{sql-error-fixing-core} skill, whose detailed repair contract is defined below. The module processes one instance at a time and uses probe-grounded evidence to decide whether to keep the current SQL, select an existing candidate, or repair the query.

\begin{promptbox}{Ambiguity-Driven SQL Probing and Repair Workflow}

\begin{enumerate}
\item Input: question, seed SQL (primary patch target), candidate SQLs, implementation difference report of candidate SQLs, schema evidence, and optional external info.

\item Generate a clarifying-question plan.
\begin{enumerate}
\item For each question, assign \texttt{question\_id}, taxonomy (\texttt{AmbiSchema}, \texttt{AmbiValue}, \texttt{AmbiIntent}), \texttt{ambiguity\_question}, \texttt{why\_high\_impact}, and \texttt{probe\_sql\_rationale}.
\item Generate exactly 8--9 questions.
\item If implementation-diff is available, exactly 3 must be diff-grounded and the rest taxonomy-grounded; otherwise all taxonomy-grounded.
\end{enumerate}

\item Generate and execute ambiguity probes.
\begin{enumerate}
\item \texttt{AmbiSchema}: ambiguous schema mappings (table/column/grain/metric).
\item \texttt{AmbiValue}: question values do not directly align with database values, leading to ambiguous filter literals.
\item \texttt{AmbiIntent}: ambiguous SQL semantics (e.g., \texttt{ORDER BY} vs.\ \texttt{GROUP BY}).
\item Write one probe SQL per question and execute all probes, saving SQL and CSV artifacts.
\end{enumerate}

\item Convert probe results into resolved assumptions.

\item Apply a minimal probe-evidence-backed SQL patch.
\begin{enumerate}
\item Patch only the seed SQL baseline; use candidates as reference only.
\item Apply changes only when supported by probe evidence; keep edits minimal and localized.
\item Use implementation diff as hypothesis context only, never as a patch source.
\item Ensure all rewrites are probe-evidence-backed (\texttt{ambiguity\_resolutions}); otherwise retain the baseline SQL.
\end{enumerate}

\item Run final sanity and freeze.
\begin{enumerate}
\item Execute final SQL with \texttt{execute\_sql.py}.
\item If execution fails or returns empty, fall back to baseline SQL.
\item Always produce \texttt{fixed\_sql/\textless instance\_id\textgreater.sql}; limit to 3 attempts.
\end{enumerate}

\end{enumerate}

\vspace{0.25em}

\textbf{Constraints}

\begin{itemize}
\item Do not read or use gold SQL or gold execution during fixing.
\item No SQL patch before probe evidence exists.
\item If implementation diff is provided, treat it as hypothesis-only context.
\item Never patch SQL directly from implementation-diff alternatives or snippets.
\item Patch assumptions must be traceable to probe outputs; if any consensus item looks inconsistent with question or schema, validate by probe before patching.
\item Probe execution and sanity checks must use \texttt{execute\_sql.py}.
\end{itemize}

\vspace{0.25em}

\textbf{Outputs}

\begin{itemize}
\item \texttt{fixed\_sql/\textless instance\_id\textgreater.sql}.
\item Probe SQL and CSV artifacts.
\item Probe report and summary files.
\item Consolidated ambiguity-resolution metadata.
\end{itemize}

\end{promptbox}

\subsection{Prompt for Detailed NL Question Generation in Synthetic Query Log Generation}
\label{sec:nl-question-generation}

We generate synthetic query log by reconstructing precise natural-language questions from analytical logic. The following prompt instructs the model to reverse-engineer detailed and unambiguous business questions that preserve the full computational semantics of the underlying SQL.

\begin{promptbox}{Prompt for Detailed NL Question Generation}

You are an expert data analyst who reconstructs precise business questions from analytical logic. Your task is to reverse-engineer a natural-language question that fully specifies the computation implied by the given analytical logic.

\vspace{0.25em}

\textbf{Inputs}
\begin{itemize}
\item Backend/Dialect: \{backend\}
\item Database ID: \{db\_id\}
\item Schema: \{references\}
\item Gold SQL: \{gold\_sql\}
\end{itemize}

\vspace{0.25em}

\textbf{Process}

Carefully reconstruct the full analytical workflow before writing the question. Identify:
\begin{itemize}
\item The exact population and all filtering conditions, including status, date, category, and NULL handling.
\item Any lookup, mapping, or normalization steps (e.g., string transformations or case normalization).
\item Entity relationships and how records are brought into scope.
\item All intermediate metrics, transformations, and timestamp or date extraction steps.
\item Ranking, scoring, segmentation, or window operations.
\item All conditional rules, including full classification logic and thresholds.
\item Aggregation logic, including grouping and multi-stage aggregation.
\item Ordering, limits, offsets, and tie-breaking rules.
\item Any excluded or unassigned records.
\end{itemize}

\vspace{0.25em}

\textbf{Requirements}

The generated question must:
\begin{itemize}
\item Fully specify the population and all filtering conditions.
\item Describe all intermediate computations and transformations.
\item Define all metrics, aggregations, rankings, and classification rules precisely.
\item Preserve all nuances of the analytical logic, including low-level details such as normalization and date extraction.
\item Clearly state the final output and what each value represents.
\item Be written as a single coherent natural-language paragraph.
\end{itemize}

\vspace{0.25em}

\textbf{Constraints}

\begin{itemize}
\item Do not omit any logic from the analytical specification.
\item Do not introduce assumptions not present in the logic.
\item Avoid vague phrases (e.g., ``analyze trends'').
\item Do not mention implementation terms (e.g., SQL, query, table, column, join).
\item Do not use bullet points or numbered lists in the output.
\item Do not include explanations or preamble text.
\end{itemize}

\vspace{0.25em}

\textbf{Output}

\begin{itemize}
\item Return only a single detailed natural-language question.
\end{itemize}

\end{promptbox}


\subsection{Prompt for Ambiguous NL Question Generation in Synthetic Query Log Generation}
\label{sec:ambiguous-nl-question-generation}

We generate realistic ambiguous user questions by relaxing fully specified analytical questions while preserving their SQL-equivalent intent. The following prompt instructs the model to produce natural text-to-SQL questions that omit details commonly left implicit by users, such as metric definitions, aggregation choices, calculation explanations, and output formatting, while retaining all hard constraints required to recover the same result.

\begin{promptbox}{Prompt for Ambiguous Text-to-SQL Question Generation}
You create realistic user questions for text-to-SQL. Given a precise detailed question and its GOLD SQL, write up to \{num\_questions\} ambiguous user-style questions that ask for the same result while omitting details real users often leave implicit, such as exact metric definitions, calculation explanations, aggregation choices, or output formatting.

\vspace{0.4em}

\textbf{Goal}
\begin{itemize}
\item Simulate natural user requests that preserve the same overall intent and result as the GOLD SQL.
\item Use the schema only to understand available concepts and entities, such as customers, orders, products, payments, categories, or events.
\item Do not mention table names or column names explicitly.
\item When appropriate, include mild schema ambiguity, such as referring to ``orders'' vs. ``items'', ``customers'' vs. ``users'', or ``products'' vs. ``categories''.
\end{itemize}

\vspace{0.4em}

\textbf{Ambiguity Patterns}
\begin{itemize}
\item Prefer missing intent keywords, such as ``sales'' without specifying revenue vs. count, ``average value'' without specifying the aggregation level, or ``top customers'' without fully explaining the ranking metric.
\item Include insufficient reasoning context, such as ``active customers'', ``customer spend'', or ``performance'' without restating all business-rule details.
\item Use mild entity-level ambiguity when natural, such as orders vs. order items, products vs. categories, or deliveries vs. shipments.
\item Leave soft output details implicit when possible, such as ordering, rounding precision, aliases, tie-breaking rules, or output formatting.
\end{itemize}

\vspace{0.4em}

\textbf{Constraint Preservation}
\begin{itemize}
\item First identify hard constraints in the GOLD SQL: filters, date ranges, ranking limits, grouping targets, aggregation scope, distinctness, conditional fallback rules, and required entity relationships.
\item Preserve every hard constraint in every generated question, even if phrased naturally or implicitly.
\item Do not remove or alter constraints that determine which records are included or how a core metric is computed.
\item Soft constraints may be omitted when they do not change the core result, including rounding precision, column aliases, output order, and explanatory calculation steps.
\end{itemize}

\vspace{0.4em}

\textbf{Precision Rules}
\begin{itemize}
\item Preserve exact time and duration logic.
\item Do not replace duration-based constraints with calendar-period phrasing; for example, ``within 7 days'' is not the same as ``within the same week''.
\item Do not introduce new filters, assumptions, entities, rankings, or business rules not implied by the GOLD SQL.
\item Avoid technical language and do not mention SQL, query, database, table, column, join, CTE, or subquery.
\end{itemize}

\vspace{0.4em}

\textbf{Inputs}
\begin{itemize}
\item Backend/Dialect: \{backend\}
\item Database ID: \{db\_id\}
\item Number of Questions: \{num\_questions\}
\item Schema: \{references\}
\item Detailed Question: \{detailed\_question\}
\item GOLD SQL: \{gold\_sql\}
\end{itemize}

\vspace{0.4em}

\textbf{Self-Check}
\begin{itemize}
\item Ensure every hard constraint is logically present in each question.
\item Ensure no question contradicts or changes the GOLD SQL.
\item Ensure no duration, date, ranking, grouping, or filter boundary has drifted.
\item Ensure a competent solver could recover the same intended result from the question.
\end{itemize}

\vspace{0.4em}

\textbf{Output}
\begin{itemize}
\item Output 1 to \{num\_questions\} lines.
\item Each line must be exactly one natural question sentence.
\item Each line must end with ``?''.
\item Output only the questions.
\item No numbering, bullets, explanations, markdown, or extra text.
\end{itemize}

\vspace{0.4em}

\textbf{Ambiguous Questions}

\end{promptbox}


\subsection{Prompt for Ambiguity-Resolution Reranking}
\label{sec:ambiguity-reranker-prompt}

We use an ambiguity-resolution reranker for NL2SQL to select the safest and most useful in-context candidate from the synthetic query log. The reranker first infers the core ambiguity slots in the question, then scores each retrieved candidate by ambiguity relevance, assumption consistency, and contradiction risk, and finally returns a compact set of candidates for downstream prompting.

\begin{promptbox}{Prompt for Ambiguity-Resolution Reranker}

You are an ambiguity-resolution reranker for NL2SQL. Given a current question and a set of retrieved candidate questions with ambiguity explanations, select up to \texttt{reranker\_top\_k} candidates that are safest and most useful as in-context ambiguity-resolution guidance.

\vspace{0.25em}

\textbf{Step 1: Analyze the Question}
\begin{itemize}
\item Infer the core ambiguity slots relevant to the question.
\item Common slot types include grouping grain or entity identifier, metric definition, counting rule, temporal conversion rule, ranking key, filtering scope, join scope or schema grounding, and null or missing-data rule.
\item For each core slot, infer the likely intended interpretation when supported by the question text.
\item If a slot cannot be confidently inferred, mark it as uncertain rather than guessing.
\item Prefer stable and schema-safe interpretations for uncertain slots.
\end{itemize}

\vspace{0.25em}

\textbf{Step 2: Evaluate Each Candidate}
\begin{itemize}
\item Judge ambiguity relevance, resolution consistency, and contradiction risk.
\item Do not reward topical similarity alone.
\item Judge assumptions, not just topic overlap.
\item Reject candidates that conflict with core assumptions or introduce harmful changes to grouping, filtering, aggregation, ranking, or final results.
\item Treat grouping-grain mismatches, counting-rule mismatches, ranking-key mismatches, metric-construction mismatches, filtering-scope mismatches, temporal-conversion mismatches, and aggregation-grain mismatches as contradictions when they can change results.
\item Prefer concrete SQL-relevant clarifications.
\item When uncertain between keep and reject, prefer reject.
\end{itemize}

\vspace{0.25em}

\textbf{Step 3: Select the Final Set}
\begin{itemize}
\item Select up to \texttt{reranker\_top\_k} candidates as a set.
\item Maximize safety, consistency, coverage of distinct core ambiguity slots, and low redundancy.
\item Prefer candidates that cover different ambiguity slots.
\item For aggregation-heavy questions, prioritize grouping, counting, and metric-definition slots.
\item For temporal or formula-heavy questions, prioritize computation formulas and filtering rules.
\item Return fewer than \texttt{reranker\_top\_k} candidates if necessary.
\end{itemize}

\vspace{0.25em}

\textbf{Output}
\begin{itemize}
\item Return valid JSON only.
\item Include question analysis, per-candidate evaluations, the selected set, and a concise selection summary.
\item Mark contradictions explicitly and provide concrete harmful assumptions when possible.
\end{itemize}

\end{promptbox}
\subsection{Prompt for LLM-as-a-Judge SQL Selection with Clarification}
\label{app:sql-selection-judge-prompt}
\begin{promptbox}{LLM-as-a-Judge Prompt for SQL Candidate Selection}

You are a SQL correction judge for text-to-SQL systems. Your job is to select the single best SQL query from exactly 10 candidates.

\vspace{0.4em}

\textbf{Task}
\begin{itemize}
    \item You will be given:
    \begin{itemize}
        \item a natural-language question,
        \item optional clarifying ambiguity questions and answers,
        \item a schema,
        \item exactly 10 candidate SQL queries.
    \end{itemize}
    \item Choose the one candidate that best satisfies the user's true intent.
    \item If clarifying question-answer pairs are provided and non-empty, treat them as the highest-priority source of intent.
    \item If no clarifying Q\&A is provided, infer the most likely intent from the question and schema alone.
\end{itemize}

\vspace{0.4em}

\textbf{Interpretation Rules}
\begin{itemize}
    \item Translate clarifying Q\&A into explicit semantic constraints, such as:
    \begin{itemize}
        \item metric definition,
        \item grouping grain,
        \item filters,
        \item join path,
        \item time window,
        \item ordering,
        \item limits,
        \item output columns.
    \end{itemize}
    \item Prefer the candidate that best matches the clarified requirements with the fewest contradictions.
    \item If multiple candidates are similar, choose the one that is most faithful to the clarified intent and schema.
    \item Use only the provided question, clarifications, schema, and candidate SQLs. Do not rely on outside knowledge.
\end{itemize}

\vspace{0.4em}

\textbf{Strict Constraints}
\begin{itemize}
    \item Select exactly one of the 10 provided candidates.
    \item Do not rewrite, edit, combine, or synthesize SQL.
    \item The output SQL must be copied verbatim from one candidate, word-for-word.
    \item Before finalizing, verify that the SQL appears exactly in the candidate list.
    \item If a candidate uses invalid tables or columns according to the schema, do not choose it unless every other candidate is worse and the issue is minor; otherwise reject it in favor of the best valid option.
\end{itemize}

\vspace{0.4em}

\textbf{Selection Procedure}
\begin{enumerate}
    \item Parse the question and any clarifying Q\&A into a precise intent.
    \item Compare all 10 candidates against that intent.
    \item Check schema validity and semantic faithfulness.
    \item Select the single best candidate.
    \item Output the candidate number and the exact SQL text, unchanged.
\end{enumerate}

\vspace{0.4em}

\textbf{Output Format}
\begin{verbatim}
Candidate: <N>
<exact SQL copied verbatim from candidate N>
\end{verbatim}

\vspace{0.4em}

\textbf{Input Fields}
\begin{itemize}
    \item Question: \{question\}
    \item Clarifying Ambiguity Q\&A: \{ambi\_blob\}
    \item Candidate SQLs (10): \{generated\_sql\_list\}
    \item Schema: \{references\}
\end{itemize}

\vspace{0.4em}

\textbf{Hard Requirement}
\begin{itemize}
    \item Output only the candidate number and the selected SQL.
    \item No explanation.
    \item No markdown.
    \item No extra text.
\end{itemize}

\end{promptbox}

\section{Examples of Derived Clarification Questions}
\label{app:derived_mcq_examples}

We show some examples of generated MCQ questions and answers from synthetic query log of AMBROSIA, Beaver, and Spider 2.0 Lite.

\begin{promptbox}{Examples of Original Questions and Derived MCQs}

\textbf{AMBROSIA.}

\textbf{Original Question:} Show all the action movies and romantic comedies lasting 2 hours.

\textbf{Derived MCQ 1:} Which interpretation best matches what you meant?
\begin{itemize}
    \item[A)] Show all 2-hour films that are classed as either action or romantic comedy.
    \item[B)] Show all action movies, and only romantic comedy films lasting 2 hours.
    \item[C)] None of the above / not what I meant.
\end{itemize}
\textbf{Selected Answer:} A.
\textbf{Explanation:} Return films that are both 2 hours long and classified as either action or romantic comedy.

\vspace{0.5em}
\textbf{BEAVER.}

\textbf{Original Question:} List the name and floor of the building with the largest floor number.

\textbf{Derived MCQ 1:} How should ``largest floor number'' be determined?
\begin{itemize}
    \item[A)] Use the floor-ordering field \texttt{FLOOR\_SORT\_SEQUENCE} to identify the top floor, and return \texttt{BUILDING\_NAME} and \texttt{FLOOR}.
    \item[B)] Use the displayed floor value \texttt{FLOOR} itself to identify the top floor, and return \texttt{BUILDING\_NAME} and \texttt{FLOOR}.
    \item[C)] None of the above / not what I meant.
\end{itemize}
\textbf{Selected Answer:} B.
\textbf{Explanation:} Compare the numeric floor labels shown in \texttt{FLOOR}, ignoring non-numeric floor values instead of using a separate sort field.

\vspace{0.5em}
\textbf{Spider 2.0 Lite.}

\textbf{Original Question:} For each visitor who made at least one transaction in February 2017, how many days elapsed between the date of their first visit in February and the date of their first transaction in February, and on what type of device did they make that first transaction?

\textbf{Derived MCQ 1:} When a visitor has multiple purchase sessions on their earliest purchase date, which session should determine the reported device type?
\begin{itemize}
    \item[A)] Use the device from the first purchase session on that date, i.e., the earliest \texttt{visitStartTime}.
    \item[B)] Use the device from any purchase session on that date, without imposing a specific time ordering.
    \item[C)] None of the above / not what I meant.
\end{itemize}
\textbf{Selected Answer:} A.
\textbf{Explanation:} Take the device type from the earliest purchase session on the first purchase day.

\end{promptbox}

\section{Prompt for MCQ from Synthetic Log Helpfulness Evaluation}
\label{app:nl2sql-clarification-helpfulness-prompt}
\begin{promptbox}{Prompt for NL2SQL Clarification Helpfulness Evaluation}
You are evaluating the utility of injected  Clarification MCQs for NL2SQL.
Assume the  Clarification MCQ has already been injected into the prompt.
\vspace{0.4em}

\textbf{Evaluation Goal}

Judge whether the injected MCQ clarification makes the intended SQL semantics
easier and more reliable for a strong NL2SQL model to recover.

\vspace{0.4em}

\textbf{Inputs}
\begin{itemize}
\item Instance ID: \{instance\_id\}
\item Original question: \{original\_question\}
\item Injected  Clarification MCQ: \{mcq\_blob\}
\item Gold SQL context: \{gold\_sql\_section\}
\end{itemize}

\vspace{0.4em}

\textbf{Decision Rule}
\begin{itemize}
\item Use the gold SQL as the primary signal for intended meaning when available.
\item Label as \texttt{helpful} if the clarification plausibly reduces the risk
of a semantically incorrect SQL query.
\item Label as \texttt{not\_helpful} if the clarification is redundant, vague,
unsupported by the gold SQL, or unlikely to change the model's SQL interpretation.
\end{itemize}

\vspace{0.4em}

\textbf{Helpful Cases}
\begin{itemize}
\item The clarification resolves or partially resolves a SQL-critical ambiguity.
\item It reduces the chance of an incorrect filter, join, aggregation, ranking,
temporal scope, or inclusion/exclusion decision.
\item It is useful for recovering the intended SQL semantics, even if incomplete.
\end{itemize}

\vspace{0.4em}

\textbf{Not Helpful Cases}
\begin{itemize}
\item The clarification mostly repeats information already clear from the question.
\item It adds no material ambiguity resolution.
\item It is too vague to guide SQL construction.
\item It is unsupported by the gold SQL.
\item It does not meaningfully change what a strong NL2SQL model would infer.
\end{itemize}

\vspace{0.4em}

\textbf{Output Requirements}
\begin{itemize}
\item Return exactly one JSON object per evaluated instance.
\item Output only the binary \texttt{helpful}/\texttt{not\_helpful} judgment.
\item Be concrete and concise.
\item Do not output multi-class labels.
\end{itemize}

\vspace{0.4em}

\textbf{Output Format}
\begin{verbatim}
{
  "helpfulness_label": "helpful",
  "gold_sql_used": true,
  "key_sql_semantic_issue": "...",
  "rationale": "..."
}
\end{verbatim}

\end{promptbox}

\section{Prompt for Ambiguity-Driven Probing Quality Analysis}
\label{app:probe-quality-analysis-prompt}

\begin{promptbox}{LLM-as-a-Judge Prompt for Ambiguity-Driven Probing Quality Analysis}

You are an evaluator for ambiguity probing reliability in text-to-SQL systems.

\vspace{0.4em}

\textbf{Evaluation Dimensions}
\begin{itemize}
\item \textbf{Probing Groundedness}
\begin{itemize}
    \item Label = 1 iff every probing ambiguity question is justified by the NLQ, implementation diff evidence, and/or ambiguity taxonomy grounding.
\end{itemize}

\item \textbf{Resolution Correctness}
\begin{itemize}
    \item Label = 1 iff every selected ambiguity resolution is logically correct and supported by probe evidence.
\end{itemize}

\item \textbf{SQL Repair Faithfulness}
\begin{itemize}
    \item Label = 1 iff the repaired SQL implements the selected resolution(s) without unrelated semantic drift.
\end{itemize}
\end{itemize}

\vspace{0.4em}

\textbf{Inputs}
\begin{itemize}
\item Instance ID: \{instance\_id\}
\item NLQ: \{question\}
\item Taxonomy Context: \{taxonomy\_context\}
\item Implementation Diff Blob: \{implementation\_diff\_blob\}
\item Ambiguity Question Plan JSON: \{ambiguity\_question\_plan\_json\}
\item Ambiguity Resolutions JSON: \{ambiguity\_resolutions\_json\}
\item Probe Report JSONL: \{ambiguity\_probe\_report\_jsonl\}
\item Before SQL: \{before\_sql\}
\item Fixed SQL: \{fixed\_sql\}
\end{itemize}

\vspace{0.4em}

\textbf{Rules}
\begin{itemize}
\item Base judgments only on the provided inputs.
\item If required evidence for a metric is missing, assign label = 0 for that metric and explain why.
\item Keep evidence concise and specific.
\item Output strict JSON only.
\end{itemize}

\vspace{0.4em}

\textbf{Output Format}
\begin{verbatim}
{
  "instance_id": "{instance_id}",
  "groundedness": {
    "label": 0,
    "grounded_questions": 0,
    "total_questions": 0,
    "precision_grounded": 0.0,
    "evidence": ["..."],
    "rationale": "..."
  },
  "resolution_correctness": {
    "label": 0,
    "correct_resolutions": 0,
    "total_resolutions": 0,
    "precision_resolution": 0.0,
    "evidence": ["..."],
    "rationale": "..."
  },
  "repair_faithfulness": {
    "label": 0,
    "required_changes_count": 0,
    "implemented_changes_count": 0,
    "precision_faithful": 0.0,
    "unrelated_drift_detected": false,
    "evidence": ["..."],
    "rationale": "..."
  },
  "all_three_pass": 0
}
\end{verbatim}

\end{promptbox}
\section{Bucketed Analysis of Hard Spider 2.0-Lite Instances}
\label{app:bucket-analysis}

\vspace{0.5em}

\begin{table}[H]
\centering
\small
\setlength{\tabcolsep}{10pt}
\renewcommand{\arraystretch}{1.15}

\begin{tabular}{lccc}
\toprule
\textbf{Difficulty Bucket} & \textbf{\# Inst.} & \textbf{Baseline} ($\boldsymbol{\mu \pm \sigma}$) & \textbf{+ Probing} \\
\midrule
Never correct    & 268 & $0.0 \pm 0.0$   & \textbf{30.6} \\
Sparsely correct & 54  & $20.4 \pm 6.2$  & \textbf{81.5} \\
\bottomrule
\end{tabular}

\vspace{0.6em}

\caption{
Execution Accuracy on difficult Spider 2.0-Lite instances using Gemma-4-31B.
To estimate baseline variability, we first run \sys ten times per instance before applying probing and group instances by the number of correct executions across runs.
Instances with zero correct executions are labeled \textit{never correct}, while instances with one to four correct executions are labeled \textit{sparsely correct}.
Baseline results report mean execution accuracy and standard deviation across the 10 runs, whereas probing is evaluated once on the resulting buckets.
}
\label{tab:spider2_bucket_probe_gain}
\end{table}

\section{Ambiguity Probing Quality Analysis}
\label{app:probe-audit}

\paragraph{Sampling strategy.}
We construct the audit set from the full 547-instance Spider~2.0 evaluation generated by SOMA-SQL with Gemma-4.
To reduce sampling bias, we use outcome-stratified proportional sampling over four transition categories:

\textit{wrong$\rightarrow$correct}: the original SQL prediction is incorrect before probing, but the repaired SQL becomes correct after ambiguity probing and repair;

\textit{correct$\rightarrow$correct}: the original SQL prediction is already correct before probing and remains correct after ambiguity probing and repair;

\textit{wrong$\rightarrow$wrong}: the original SQL prediction is incorrect before probing and remains incorrect after ambiguity probing and repair;

\textit{correct$\rightarrow$wrong}: the original SQL prediction is correct before probing, but ambiguity probing and repair introduce an incorrect final SQL prediction.

We sample 50 instances proportionally from these categories and additionally require artifact completeness, including the probing plan, ambiguity resolutions, and probe report.

\paragraph{Metric definitions.}
Each sampled instance is evaluated using binary labels on three dimensions:
(1) \textbf{Probing groundedness},
(2) \textbf{Resolution correctness}, and
(3) \textbf{SQL repair faithfulness}.

\textbf{Probing groundedness} measures whether probing questions are justified by the NLQ and supported by either implementation-diff evidence or ambiguity taxonomy grounding (\textit{AmbiSchema}, \textit{AmbiValue}, or \textit{AmbiIntent}).

\textbf{Resolution correctness} measures whether the selected ambiguity resolutions are supported by probe evidence and consistent with the intended user meaning.

\textbf{SQL repair faithfulness} measures whether the repaired SQL correctly implements the selected resolution(s) without introducing unrelated semantic drift.

For each metric, a label of 1 is assigned only if all relevant sub-items within the instance satisfy the criterion; otherwise, the instance receives label 0.

\paragraph{Judge configuration.}
We use a separate judge model (gpt5.4) from the generation model (\texttt{gemma-4}) to reduce self-enhancement bias \cite{ye2024justice} from LLM as a judge.
The judge is required to produce structured JSON outputs containing binary labels and supporting rationale for each metric. The full judge prompt is provided in \ref{app:probe-quality-analysis-prompt}.

\section{End-to-End Example: Ambiguity-Guided SQL Repair}
\label{app:probe-example}
\subsection{Task Overview}
\label{app:probe-task}

\begin{promptbox}{}

\textbf{Question}: \textit{show me churn for enterprise customers last quarter}
\vspace{0.4em}

\textbf{Ambiguity}:
The term ``churn'' may correspond to multiple business definitions:
\begin{itemize}
\item canceled subscriptions
\item inactive accounts
\item lost revenue
\item customers with no purchases
\end{itemize}

\vspace{0.4em}

\textbf{Ambiguity Hypothesis (Implementation Diff)}

\begin{verbatim}
slot: churn_definition

candidates: {
  canceled_subscription,
  inactive_account,
  revenue_loss,
  no_recent_purchase
}

trigger: "churn" mention in NL
\end{verbatim}

\end{promptbox}

\vspace{0.5em}

\noindent
This example demonstrates how \sys resolves semantic ambiguity using database-grounded probing.

\subsection{Repair Prompt (Condensed)}

\label{app:probe-prompt}

\begin{promptbox}{}

\begin{verbatim}
Use $sql-ambiguity-driven-probing to fix demo_churn_001.

Question text available: true
Question: show me churn for enterprise customers last quarter

Candidate SQL count from run_records: 10
Candidate run names: run_1,...,run_10
Candidate SQL context: multiple candidate SQL variants available

Implementation diff available: true
implementation_diff_blob:
  slot = churn_definition
  candidates = {
    canceled_subscription,
    inactive_account,
    revenue_loss,
    no_recent_purchase
  }
  ambiguity trigger = "churn"

Probe policy:
- Create 8–9 clarification probes total
- Exactly 3 probes grounded in implementation diff
- Remaining probes grounded in taxonomy

Constraints:
- Treat implementation diff as ambiguity hypotheses only
- Do not patch directly from diff
- Patch only using probe evidence

Execution policy:
- Execute probes against the database
- Use database results as grounding evidence

Repair objective:
- Apply minimal, probe-supported SQL changes
- Preserve original semantics unless evidence supports change
\end{verbatim}

\end{promptbox}

\subsection{Ambiguity Probing Report}
\label{app:probe-report}
\begin{promptbox}{}

\textbf{Probe 1: Canceled subscription interpretation}
\begin{verbatim}
SELECT COUNT(*) AS churned_customers
FROM subscriptions s
JOIN customers c ON s.customer_id = c.customer_id
WHERE c.segment = 'enterprise'
  AND s.status = 'canceled'
  AND s.cancel_date >= date('now','start of month','-3 month')
  AND s.cancel_date < date('now','start of month');

Result:
[[4821]]
\end{verbatim}

\textbf{Probe 2: Inactive account interpretation}
\begin{verbatim}
SELECT COUNT(*) AS inactive_customers
FROM customers c
WHERE c.segment = 'enterprise'
  AND c.last_activity_date <
      date('now','start of month','-3 month');

Result:
[[137]]

\end{verbatim}

\textbf{Probe 3: No recent purchase interpretation}
\begin{verbatim}
SELECT COUNT(DISTINCT c.customer_id)
FROM customers c
LEFT JOIN orders o
  ON c.customer_id = o.customer_id
WHERE c.segment = 'enterprise'
GROUP BY c.customer_id
HAVING MAX(o.order_date) <
       date('now','start of month','-3 month');

Result:
[[94]]
\end{verbatim}
\end{promptbox}

\subsection{Resolution}
\label{app:probe-resolution}

\begin{promptbox}{Resolved Ambiguity Interpretation}

\textbf{Resolved Interpretation}

``Churn'' refers to \textbf{canceled subscriptions}.

\vspace{0.4em}

\textbf{Evidence}

Canceled subscriptions produce substantially stronger and temporally consistent signals compared to alternative interpretations.

\vspace{0.4em}

\textbf{Why Alternatives Fail}

Inactive-account and no-purchase interpretations yield sparse evidence and do not align with observed subscription behavior.

\vspace{0.4em}

\textbf{Decision Rule}

Select the interpretation with the strongest database-grounded evidence.

\vspace{0.4em}

\textbf{Repair Action}

Apply filtering on subscription cancellation status and cancellation timestamps.

\end{promptbox}

\subsection{Final Repaired SQL}
\label{app:probe-sql}

\begin{promptbox}{}

\begin{verbatim}
SELECT
  COUNT(DISTINCT s.customer_id) AS churned_enterprise_customers
FROM subscriptions s
JOIN customers c
  ON s.customer_id = c.customer_id
WHERE c.segment = 'enterprise'
  AND s.status = 'canceled'
  AND s.cancel_date >= date('now','start of month','-3 month')
  AND s.cancel_date < date('now','start of month');
\end{verbatim}

\end{promptbox}

\end{document}